\documentclass[10pt,journal,twocolumn]{IEEEtran}
\usepackage{cite}
\usepackage{fancyhdr}
\usepackage{float}
\usepackage{graphicx}
\usepackage{epsfig}
\usepackage{bm}
\usepackage{amsmath}
\usepackage{mathrsfs}
\usepackage{subfigure}
\usepackage{array}
\usepackage{booktabs}
\usepackage{amssymb}
\usepackage{graphicx}
\usepackage{epstopdf}
\usepackage{algorithm}
\usepackage{amsthm}
\usepackage{bbm}
\usepackage{graphicx}
\usepackage{epsfig}
\usepackage{color}
\usepackage{cite}

\newtheorem{theorem}{Theorem}
\newtheorem{lemma}{Lemma}
\newtheorem{prop}{Proposition}
\newtheorem{defn}{Definition}
\newtheorem{remark}{Remark}

\makeatother

\usepackage{algpseudocode}
\usepackage{amsmath}

\title{RDP-GAN: A  R\'{e}nyi-Differential Privacy based Generative Adversarial Network}
\begin{document}
\author{\IEEEauthorblockN{Chuan Ma, \emph{Member, IEEE}\thanks{C. Ma, J. Li and K. Wei are with the School of Electronic and Optical Engineering, Nanjing University of Science and Technology, Nanjing, China. (e-mail: \{chuan.ma, jun.li, kang.wei\}@njust.edu.cn).},
Jun Li, \emph{Senior Member, IEEE},
Ming Ding, \emph{Senior Member, IEEE},\\
\thanks{M. Ding is with Data61, CSIRO, Australia (e-mail: Ming.Ding@data61.csiro.au).}
Bo Liu, \emph{Senior Member, IEEE},
\thanks{B. Liu is with University of Technology Sydney, NSW 2007, Australia (email:
bo.liu@uts.edu.au).}
Kang Wei, \emph{Student Member, IEEE},\\
Jian Weng, \emph{Member, IEEE},
\thanks{J. Weng is with Jinan University, Guangzhou 510632, China (email: cryptjweng@gmail.com).}
and H. Vincent Poor, \emph{Fellow, IEEE}
\thanks{H. V. Poor is with the Department of Electrical Engineering, Princeton University, Princeton, NJ 08544 USA (e-mail: poor@princeton.edu).}
}}

\maketitle
\begin{abstract}
Generative adversarial network (GAN) has attracted increasing attention recently owing to its impressive ability to generate realistic samples with high privacy protection. Without directly interactive with training examples, the generative model can be fully used to estimate the underlying distribution of an original dataset while the discriminative model can examine the quality of the generated samples by comparing the label values with the training examples. However, when GANs are applied on sensitive or private training examples, such as medical or financial records, it is still probable to divulge individuals' sensitive and private information. To mitigate this information leakage and construct a private GAN, in this work we propose a  R\'{e}nyi-differentially private-GAN (RDP-GAN), which achieves differential privacy (DP) in a GAN by carefully adding random noises on the value of the loss function during training. Moreover, we derive the analytical results of the total privacy loss under the subsampling method and cumulated iterations, which show its effectiveness on the privacy budget allocation. In addition, in order to mitigate the negative impact brought by the injecting noise, we enhance the proposed algorithm by adding an adaptive noise tuning step, which will change the volume of added noise according to the testing accuracy. Through extensive experimental results, we verify that the proposed algorithm can achieve a better privacy level while producing high-quality samples compared with a benchmark DP-GAN scheme based on noise perturbation on training gradients.
\end{abstract}

\begin{IEEEkeywords}
Generative Adversarial Network, R\'{e}nyi-Differential Privacy, Adaptive Noise Tuning Algorithm
\end{IEEEkeywords}

\section{Introduction}
Recent technological advancements are transforming the ways in which data are created and processed. With the advent of the Internet-of-things (IoT), the number of intelligent devices in the world is rapidly growing in the last couple of years. Many of these devices are equipped with various sensors and increasingly powerful hardware, which allow them to not just collect, but more importantly, process data at unprecedented scales.
In the concurrent development, artificial intelligence (AI) has revolutionized the ways that information is utilized with ground breaking successes in areas such as computer vision, natural language processing, voice recognition, etc \cite{wang2018edge}. Therefore, there is a high demand for harnessing the rich data provided by distributed devices to improve machine learning models.

In the past few years, deep learning has demonstrated largely improved performance over traditional machine learning methods in various applications, e.g., image understanding \cite{He_2016_CVPR}, speech recognition \cite{6638947}, cancer analysis \cite{6977954}, and the game of GO \cite{silver2016mastering}. The great success of deep learning is owing to the development of powerful computing processor and the availability of massive data for training the neural networks.
However, there exists domains where the accessibility of this huge data is not fully granted. For example, the sensitive medical data are usually not open-access in most countries. Thus, building a high-quality analytical model remains to be challenging at present.
At the same time, data privacy has become a growing concern for clients.  In particular, the emergence of centralized searchable data repositories has made the leakage of private information an urgent social problem, e.g., health conditions, travel information, and financial data. Furthermore, the diverse set of open data applications, such as census data dissemination and social networks, place more emphasis on privacy concerns. In such practices, the access to real-life datasets may cause information leakage even in pure research activities. Consequently, privacy preservation has become a critical issue.

Fortunately, generative models \cite{makhzani2015adversarial} provide us with a promising solution to alleviated the data scarcity issue. By sketching the data distribution from a small set of training data, it is feasible to sample from the input distribution and further generate more synthetic samples. By combining the complexity of deep neural networks and game theory, the Generative Adversarial Networks (GANs) \cite{goodfellow2014generative} and its variants have demonstrated impressive performance on modelling the underlying data distribution, generating high quality ``fake" samples that are hard to be distinguished from real ones, i.e., synthetic samples. In this way, the availability of GANs can fully facilitate the generated data and well protect the privacy of individuals.

However, the GANs can still implicitly disclose private information on the training examples. The adversarial training procedure and the high model complexity of deep neural networks, jointly encourage a distribution that is concentrated around training samples. By repeated sampling from the distribution, there is a considerable chance of recovering the training examples. For example, the authors in \cite{hitaj2017deep} introduced an active inference attack model that can reconstruct training samples from sampling generated data. Therefore, it is highly demanded to have generative models that not only generates high quality samples but also protects the privacy of the training data. Indeed, a private GAN has its potential to address this challenging problem.

To preserve the data privacy, one way is to make the synthetic data differentially private with respect to the original data. To do this, the authors in \cite{jordon2018pate} modified the training procedure of the discriminator to be differentially private by using the Private Aggregation of Teacher Ensembles (PATE) framework. The post-processing theorem then guarantees that the generator of GAN will also be differentially private. Unlike \cite{jordon2018pate}, the work in \cite{wang2018subsampled} introduced a subsampled randomized algorithm that first takes a subsample for the dataset generated through a subsampling procedure, and then applies a known randomized mechanism $\mathcal{M}$ on the subsampled data points. It is important to exploit the randomness in subsampling because if the original mechanism is differentially private, then the subsampled mechanism also obeys differentially privacy, referred to the ``privacy amplification lemma". To further conduct the sampling rate,
\cite{xie2018differentially,8636556} introduced a key idea that adding random noise to the updating gradients of the discriminator during training, and provided privacy guarantees. Moreover, norm clipping is used to bound the parameters while it will also control the maximum influence of added noises. However, bounding on the added noise will always violate the privacy level, as its maximum scale of noise is limited to one in their assumptions. To keep track of privacy loss, differential privacy (DP) that measures the difference in output between two input databases differing by at most one element, has been evolved in \cite{dwork2016concentrated}. By ensuring each gradient descent step is differentially private, the final output model satisfies a certain level of DP given by the strong composition property. Moreover, to obtain a tighter estimation, Abadi et al. \cite{abadi2016deep} proposed the moment accountant methods, which track the log moments of the privacy loss variable under random sampling.

Nevertheless, the training of GAN usually suffers from its instability, and directly adding random noises on the updating gradients will definitely harm the learning performance. Therefore, there needs a more subtle design on the privacy protection scheme. In addition, the architecture of GAN evolves multiple iterations. In more details, the fashion Wasserstein GAN \cite{arjovsky2017wasserstein} will train the generator multiple times and train the discriminator once to obtain better results. Thus, to conduct the privacy analysis, it is necessary to give an accurate estimation on each iteration.
To sum up, there are three key tasks to solve in the design of a private GAN:
\begin{itemize}
  \item Design a differentially private algorithm for GAN while guaranteeing its performance.
  \item Obtain a tight privacy loss estimation under subsampling.
  \item Obtain an accurate privacy budget calculation for cumulated iterations.
\end{itemize}

Accordingly, in this work, we propose a differentially private GAN. To accurately estimate the privacy loss of the proposed algorithm, we elaborate on the definition of R\'{e}nyi-DP which can achieve a tighter privacy analysis compared to the moment accountant method. Furthermore, we improve the proposed algorithm by designing an adaptive noise tuning algorithm, which can achieve better learning performance. Specifically, the contributions of this work are listed as follows:

\begin{itemize}
  \item To achieve a differentially private GAN, we propose a method, where random noises are added on the value of the loss function of a discriminator in each iteration. Different from previous works in \cite{xie2018differentially,8636556}, the proposed algorithm does not need an additional norm clipping method as the chosen loss function can be bounded naturally.
  \item To investigate the privacy loss of the proposed algorithm, we first elaborate on the effectiveness of R\'{e}nyi-DP in the subsampling step, which is proved to achieve a tighter analysis compared with the moment accountant method. Then we obtain the expression of the total privacy cost for cumulated iterations.
  \item To obtain better learning performances, we further improve the proposed algorithm by adding an adaptive noise tuning step, called the RDP-ANT algorithm. The improved algorithm can adaptively change the value of added noise according to the testing accuracy, and eventually consummate the algorithm.
  \item We conduct sufficient numerical experiments to verify the effectiveness of the analytical results. Compared with other algorithms, the proposed algorithm will inject less noise and obtain better performance in turn when achieving the same privacy level. Moreover, we also show the superiority of the RDP-ANT algorithm compared with other noise decay methods.
\end{itemize}

The rest of the paper is structured as follows.
We introduce the back ground in Section II,
and propose the RDP-GAN algorithm with comprehensive privacy loss analysis in Section III.
In Section IV, we show the details of the improved adaptive noise tuning algorithm.
Then, we validate the privacy analysis through extensive simulations and discuss the network performance in Section V,
and conclude the paper in Section VI.
\section{Preliminary}
In this section, we present some background knowledge of DP and GAN.
\subsection{($\epsilon, \delta$)-Differential Privacy}
We first recall the standard definition of ($\epsilon, \delta$)-DP \cite{dwork2014algorithmic}.
\begin{defn}
($\epsilon, \delta$-DP). A randomized mechanism $f: \mathcal{X} \mapsto \mathcal{R}$ offers ($\epsilon, \delta$)-DP if for any adjacent $X, X' \in \mathcal{X}$ and $S \subset \mathcal{R}$,
\begin{equation}
\Pr \left[ {f\left( X \right) \in S} \right] \le {e^\epsilon }\Pr \left[ {f\left( {X'} \right) \in S} \right] + \delta,
\end{equation}
where $f(X)$ denotes a random function of $X$.
\end{defn}

The above definition is contingent on the notion of adjacent inputs $X$ and $X'$, which is domain-specific and is typically chosen to capture the contribution to the mechanism's input by a single individual. Moreover, to avoid the worst-case scenario of always violating privacy of a $\delta$ fraction of the dataset, the standard recommendation is to choose $\delta \ll 1/N$ or even $\delta = \textrm{negl}(1/N)$, where $N$ is the data size.

The definition of ($\epsilon, \delta$)-DP is initially proposed to capture privacy guarantees of the Gaussian mechanism, defined as follows:
\begin{equation}
{\textbf{G}_\sigma }f(X) \buildrel \Delta \over = f(X) + \mathcal{N}(0,{\sigma ^2}),
\end{equation}
where $\mathcal{N}(0,\sigma ^2)$ represents a random noise, which follows the Gaussian normal distribution with mean $0$ and standard derivation $\sigma$.
Elementary analysis shows that the Gaussian mechanism satisfies a continuum of incomparable ($\epsilon, \delta$)-differential privacy \cite{dwork2014algorithmic}, for all combinations of $\epsilon < 1$ and $\sigma > {\Delta _2}f\sqrt {2\ln 1.25/\delta } /\epsilon $, where the $l_2$-sensitivity of $f$ is defined as
\begin{equation}
{\Delta _2}f \buildrel \Delta \over = \mathop {\max }\limits_{X,X' \in \mathcal{X}} {\left\| {f(X){{ - f(X')}}} \right\|_2},
\end{equation}
and taken over all adjacent inputs $X$ and $X'$.

We also list the proposition that will be used in the proof.
\begin{prop} \label{pp}
(Post-processing.) Let $\mathcal{M}$ be an ($\epsilon, \delta$)-differentially private algorithm and let $\mathcal{F} : \mathcal{O} \rightarrow \mathcal{O}'$ where $O'$ is any arbitrary space. Then $\mathcal{F} \circ \mathcal{M}$ is $(\epsilon, \delta)$-differentially private.
\end{prop}

\subsection{$\epsilon$-R\'{e}nyi Differential Privacy}

The original DP has been investigated in several works to estimate the privacy level of a typical machine learning algorithm, such as Deep Learning \cite{abadi2016deep}, Federated Learning \cite{wei2020performance} and GAN \cite{xie2018differentially,8636556}. However, the original DP needs strong assumption and will damage its accuracy when using composition theorem, as pointed in \cite{mironov2017renyi}. Thus, we
elaborate a generalization of the notion of DP based on the concept of the R\'{e}nyi divergence, which can avert inaccuracy. The R\'{e}nyi divergence is classically defined as follows:
\begin{defn}
(R\'{e}nyi divergence). For two probability distributions $P$ and $Q$ defined over $\mathcal{R}$, the R\'{e}nyi divergence of order $\alpha > 1$ is
\begin{equation}
{D_\alpha }\left( {P||Q} \right) \buildrel \Delta \over = \frac{1}{{\alpha  - 1}}\log {E_{x \sim Q}}{\left( {\frac{{P\left( x \right)}}{{Q\left( x \right)}}} \right)^\alpha },
\end{equation}
where $P(x)$ denotes the density of $P$ at $x$.
\end{defn}
It motivates exploring a relaxation of DP based on the R\'{e}nyi divergence.
\begin{defn}
(($\alpha, \epsilon$)-RDP). A randomized mechanism $f: \mathcal{X} \mapsto \mathcal{R}$ is said to have $\epsilon$-R\'{e}nyi DP of order $\alpha$, or ($\alpha, \epsilon$)-RDP for short, if for any adjacent $X, X' \in \mathcal{X}$ it holds that
\begin{equation} \label{RDP}
{D_\alpha }\left( {f\left( X \right)||f\left( {X}' \right)} \right) \le \epsilon.
\end{equation}
\end{defn}
\begin{remark}
Similar to the definition of DP, a finite value for $\epsilon$-RDP implies that feasible outcomes of $f(X)$ for some $X\in \mathcal{X}$ are feasible, i.e., have a non-zero density, for all inputs from $\mathcal{X}$ except for a set of measure 0. Assuming that this is in cause, we let the event space be the support of the distribution.
\end{remark}

We then list some propositions that will be used in deriving the analytical results in the following.
\begin{prop} \label{comp}
Let $f: \mathcal{D} \mapsto \mathcal{R}_1$ be $(\alpha, \epsilon_1)$-RDP and $g: \mathcal{R}_1 \times \mathcal{D} \mapsto \mathcal{R}_2$ be $(\alpha, \epsilon_2)$-RDP, then the mechanism defined as $(X,Y)$, where $X\sim f(\mathcal{D})$ and $Y \sim g(X,\mathcal{D})$, satisfies $(\alpha,\epsilon_1+\epsilon_2)$-RDP.
\end{prop}

\begin{prop} \label{privacy}
(Probability preservation). Let $\alpha > 1$, $P$ and $Q$ be two distributions defined over $\mathcal{R}$ with identical support. $A \subset \mathcal{R}$ be an arbitrary event. Then
\begin{equation}
P\left( A \right) \le {\left\{ {\exp \left[ {{\mathcal{D}_\alpha }(P||Q)} \right] \cdot Q\left( A \right)} \right\}^{\left( {\alpha  - 1} \right)/\alpha }}.
\end{equation}
\end{prop}

\subsection{Generative Adversarial Networks}
GAN is a recently developed generative model to produce synthetic images or texts after being trained \cite{goodfellow2014generative}. The learning process in the model is based on one generator (g) and one discriminator (d) neural networks playing in the following zero-sum minimax (i.e., adversarial) game:
\begin{equation}
\begin{split}
&\mathop {\min }\limits_G \mathop {\max }\limits_D V(g,d) = \\
&\mathbb{E}{\left[ {\log (d(x)} \right]_{{\rm{x}} \sim {{{p}}_{data}}(x)}} + \mathbb{E}{\left[ {\log \left( {1{\rm{ - }}d\left( {g\left( {\rm{z}} \right)} \right)} \right)} \right]_{z \sim p(z)}},
\end{split}
\end{equation}
where $p(z)$ is a prior distribution of latent vector $z$, $g(\cdot)$ is a generator function, and $d(\cdot)$ is a discriminator function whose output spans $[0,1]$. $d(x) = 0$ (resp. $d(x)=1$) indicates that the discriminator $d$ classifies a sample $x$ as generated (resp. real).

In Fig.~\ref{GAN}, we show a typical structure of GAN. $g$ and $d$ can be any form of neural networks. The discriminator  attempts to maximize the objective, whereas the generator attempts to maximize the objective. In other words, the discriminator attempts to distinguish between real and generated samples, while the generator  attempts to generate real-looking fake samples that the discriminator  cannot distinguish from real samples. According to the feedback in the interaction process, the generator will updates its network to enhance the quality of the generated samples. Adversaries exist in this environment and may attack the interaction process, and obtain private information from clients. In order to avoid this privacy leakage, in the next section we propose a method to improve the privacy performance of GAN.

\begin{figure}
\centering
  \includegraphics[width=0.4\textwidth]{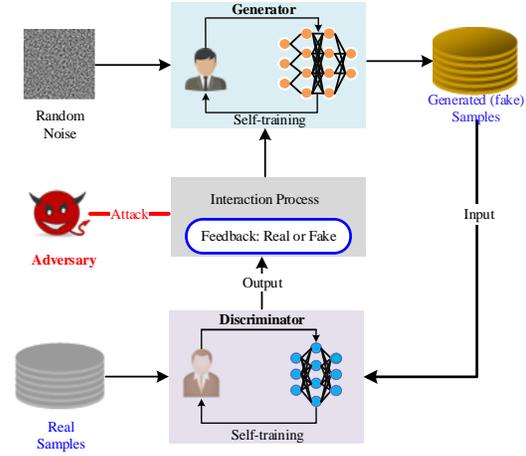}
  \caption{The structure of GAN} \label{GAN}
\end{figure}

\section{Design of RDP-GAN}
In this section, we elaborate the design of RDP-GAN, a R\'{e}nyi differentially private adversarial network to mitigate information leakage, and maintain desirable utility in the generated data.

\subsection{RDP-GAN Framework}
In most real-world problems, the ture data distribution $p(x)$ is unknown and needs to be estimated empirically. Since we are primarily interested in data synthesis, and we will use GANs as the mechanism to estimate $p(x)$ and draw samples from it. If trained properly, a RDP-GAN will mitigate inference attack during the data analysis.

The framework of RDP-GAN is structured as follows. Sensitive data $X$ is first fed into a discriminator  with a privacy-preserving layer. This discriminator is used to train a differentially private generator to produce a private artificial dataset $\tilde{X}$. Different from the work on differentially private deep learning (e.g., \cite{abadi2016deep,xie2018differentially,8636556,yu2019differentially}), the proposed algorithm achieves DP by injecting random noise on the value of the loss function in the interaction procedure. In more details, the rational behind our design is as follows:
\begin{itemize}
  \item Different from existing works, which add noise on the gradients, the proposed algorithm focuses on the design on the interaction between the generator and discriminator. This is because that from the aspect of an adversary, the interaction process is more vulnerable to attack than the network architecture if the generator and discriminator are not assembled. Moreover, it is not trivial to acknowledge information from the inside of the discriminator or generator.
  \item Adding noise on the value of the loss function is a direct method compared to the method in \cite{xie2018differentially,8636556,yu2019differentially}, which will bring explicit privacy protection. This is because that the loss values, not the parameters, are exchanged between generator and discriminator, and the perturbation on the loss value will directly influence the updates of the generator. In addition, the privacy estimation is not that accuracy as the changes of the privacy level from parameters to loss values are ignored in the previous works.
  \item Adding noise on the value of the loss function does not need an extra norm clipping function as we naturally use a bounded activation function as the last layer of neuron network. Manually bounding the loss function is an non-trivial task as its value is sensitive and needs adjustment based on empirical results.
\end{itemize}

Despite the fact that the generator does not have access to the real data $X$ in the training process, one cannot guarantee DP because of the information passed through with the loss value from the discriminator. A simple high level example will illustrate such breach of privacy. Let the dataset $X$, $X'$ contain some small real numbers. The only difference between these two datasets is the number $x' \in X'$, which happens to be extremely large. Since the impact of $x'$ is large enough to influence the performance of the discriminator, specifically on its loss value, it will lead to a large difference whether $x$ is used to train or not. In this case, this difference on the loss value will be propagated to the generator that will break privacy.

\subsection{The Implementation of RDP-GAN}
Our method focuses on preserving the privacy during the training procedure, and we add noise on the loss value of the discriminator as follows:
\begin{equation}
\mathbb{E}{\left[ {\log (d(x)} \right]{{}  {}}}=F_d(W;X)\leftarrow \sum ^ {m}_{i=1}F_d({w_i};x_i)+\mathcal{N}(0,\sigma^2),
\end{equation}
where we use $F_d(;)$ to represent the value of the loss function of all data samples, and $m$ denotes the sample size of the discriminator, respectively.

We first explain this perturbation on the loss value will not vanish during back propagation. Considering a $L$ layers network, the perturbed loss value of the last layer (the $L$-th layer) can be represented by
\begin{equation} \label{u1}
F_d^L=\nabla_A C\odot A'(z^L)+\mathcal{N}(0,\sigma^2),
\end{equation}
where $A$ denotes the activation function of the $L$-th layer, $C$ denotes the correct label, $\odot$ expresses the XNOR function that outputs a logical one or true only if two inputs are the same, and $z^L$ is the input of the $L$-th layer, respectively \cite{ruder2016overview}. During back propagation, this perturbation is involved as follows:
\begin{equation} \label{u2}
F_d^l=[(w^{l+1})^TF_d^{l+1}]\odot A'(z^l),
\end{equation}
where $l=L-1,L-2,...,2,1$. Thus, if adding noise on the value of the discriminator can achieve DP, this perturbation will be elaborated during the back propagation when updating the discriminator, and ensures that the discriminator is DP.
Moreover, the loss value (parameters) of generator can also guarantee DP because of post-processing property in \cite{dwork2014algorithmic} which states that any mapping (operation) after a differentially private output will not invade the privacy. Here the mapping is in fact the computation of parameters of generator.

The RDP-GAN procedure is summarized in Algorithm~\ref{gan1}.

  \begin{algorithm}[htb]
  \caption{ Framework of the Proposed RDP-GAN.}
  \label{gan1}
  \begin{algorithmic}[1]
    \Require
      real samples: $\{x_1,x_2, \cdots\}\sim p(x)$; prior samples: $\{z_1,z_2, \cdots\}\sim p(z)$; sampling rate: $q$; number of discriminator iterations per generator iteration: $n_d$; number of generator iteration: $n_g$; noise scale: $\sigma$;
    \Ensure
      differential private generator.
    \State Initialize discriminator parameters $d_0$, generator parameters $g_0$.

    \State \textbf{For} $t_1 =1, . . ., n_g$ do

    \State ~~\textbf{For} $t_2 = 1, . . ., n_d$ do

    \State ~~~~Sample $\{x_1,x_2, \cdots x_m\}\sim p(x)$ a batch of $m$ real data points.

    \State ~~~~Sample $\{z_1,z_2, \cdots z_m\}\sim p(z)$ a batch of $m$ prior samples.

    \State ~~~~Train the discriminator $d$.

    \State ~~~~For each data sample $i$, add a random noise on the value of the loss function as $F_d(W;X)\leftarrow \sum ^ {m}_{i=1}F_d^{t_2}({w_i};x_i)+\mathcal{N}(0,\sigma^2)$.

     \State ~~~~Update the discriminator $d$ according to Eq.~(\ref{u1}) and (\ref{u2}).


    \State ~~\textbf{End for}

    \State ~~Update the generator $g$.

    \State \textbf{End for}

    \State Return differentially private generator $g$.
  \end{algorithmic}
\end{algorithm}
\subsection{Privacy Analysis of RDP-GAN}
Because the proposed algorithm consists of multiple iterations and in order to analyze the total privacy loss, we first prove that adding noise on the value of the loss function in the discriminator ensures the R\'{e}nyi differential privacy, and then show that the generator can satisfy DP.
\subsubsection{Sensitivity on the loss function}
According to the definition of DP, we consider two adjacent database $X$ and $X'$ with same size, and only differ by one sample.

Consequently, for the discrimination in one iteration, the training process can be written as:
\begin{equation}
\mathcal{L}_d(X) = \mathop {min}\limits_W F_d(W;X),
\end{equation}
where $\mathcal{L}_d(X)$ denotes the training process. Therefore, the sensitivity on the loss function can be expressed as \begin{equation}
\Delta S = \mathop {\max }\limits_{X,X' \in \mathcal{X}} ||F_d(W;X) - F_d(W';X')||.
\end{equation}

To make the loss function bounded, we can choose a bounded activation as the last network layer, thus the value ${W_i\left( {{X_i}} \right)}$ of the $i$-th sample is bounded. So the value of the loss function $F(;)$ has a bounded input with unchanged label, and its output is naturally bounded according to the forward processing property. In addition, we denote the bound of loss function by $C$\footnote{Note here $C$ can be different for different datasets and loss functions. In our experimental results, we will conduct experiments to obtain $C$ empirically.}. Thus, following the same logic of \cite{wei2020performance}, the sensitivity of loss function can be expressed as
\begin{equation}\label{eq:sen}
\Delta S=C/|X|.
\end{equation}

From Eq.~(\ref{eq:sen}), we can observe that when the batch size increases, the sensitivity should decrease accordingly. The intuition is that the maximum difference on single data sample can make is set to $C$ and such a maximum difference is scaled down by a factor of $|X|$ since every data sample contributed equally to the resulting model. For example, if a model is trained by patient records with a certain disease, acquiring that an individual's record is among them directly affects the privacy. If we increase the size of the dataset, it will be more difficult to determine whether this record is part of the training data or not.
\subsubsection{Privacy analysis on the discriminator}
With the sensitivity, we can design the Gaussian mechanism $\mathcal{N}(0,\sigma^2)$ with $(\alpha,\epsilon)$-RDP requirement in terms of the sampling rate $q$ under $n_d$ iterations. In Theorem~\ref{22}, we present the privacy level of the output of discriminator as follows.

\begin{theorem} \label{22}
Given the sampling rate $q$ and the sensitivity $\Delta S$, adding random noise $\mathcal{N}(0,\sigma^2)$ on the loss function, the output of discriminator for each iteration can satisfy $(\alpha,\epsilon)$-RDP, and $\epsilon \leq \frac{q\alpha^2\Delta S^2}{2(\alpha-1)\sigma^2}$.
\end{theorem}

\begin{proof}
Please find the proof in the Appendix A.
\end{proof}

We then provide the following lemma, which will be used to prove the output of discriminator after $n_d$ iterations can satisfy $(\epsilon_d,\delta)$-DP. For simplicity, we rewrite the loss function $F_d(W;X)$ as $F_d(X)$ in the following.
\begin{lemma} \label{33}
Let $f: \mathcal{D} \mapsto \mathcal{R}$ be an adaptive composition of $n$ mechanisms all satisfying $(\alpha,\epsilon)$-RDP. Let $X$ and $X'$ be two adjacent inputs. Then for any subset $S\subset \mathcal{R}$:
\begin{equation} \label{11111}
\begin{split}
&\Pr \left[ {F_d(X) \in S} \right] \le \\
&\exp \left\{ {2\epsilon \sqrt {n\log /\Pr \left[ {F_d(X') \in S} \right]} } \right\} \cdot \Pr \left[ {F_d\left( {X'} \right) \in S} \right].
\end{split}
\end{equation}
\end{lemma}

\begin{proof}
Please find the proof in the Appendix B.
\end{proof}
\begin{remark}
Compared with the standard composition theorem, i.e., the Proposition~\ref{comp}, Lemma~\ref{33} can achieve a tighter privacy estimation on multiple mechanisms.
\end{remark}
\begin{theorem} \label{44}
Given the number of discriminator iterations $n_d$, the sampling rate $q$ and the sensitivity $\Delta S$, adding random noise $\mathcal{N}(0,\sigma^2)$ on the value of the loss function of discriminator, the output guarantees $(\epsilon_{{\emph{d}}},\delta)$-DP under a composition of $n_d$ RDP mechanisms, when $\log \left( {1/\delta } \right) \ge {\epsilon ^2}n_d$. The expression of $\epsilon_{\emph{d}}$ can be derived as
\begin{equation}
\epsilon_{\emph{d}} \buildrel \Delta \over = 4\epsilon \sqrt {2n_d\log \left( {1/\delta } \right)},
\end{equation}
where $\epsilon = \frac{q\alpha^2\Delta S^2}{2(\alpha-1)\sigma^2}$.
\end{theorem}

\begin{proof}
Please find the proof in the Appendix C.
\end{proof}
\subsubsection{Privacy analysis of the generator}
We next provide the analytical privacy results on the generator according to the post processing property.
\begin{theorem} \label{11}
The output of generator in the proposed algorithm guarantees $(\epsilon_{{\emph{g}}},\delta)$-DP, where
\begin{equation}
\epsilon_{\emph{g}}=4 \frac{q\alpha^2\Delta S^2}{2(\alpha-1)\sigma^2} \sqrt {2n_d\log \left( {1/\delta } \right)}
\end{equation}
\end{theorem}

\begin{proof}
The privacy mechanism guarantees a direct consequence from followed Proposition~\ref{pp} by the post processing property of DP.
\end{proof}

\subsubsection{Privacy analysis of the total training process}
Finally, we come to our final results that the proposed algorithm after $n_g$ iterations satisfies the DP guarantee.
\begin{theorem} \label{total}
The proposed algorithm satisfies $(\epsilon_{total},\delta)$-DP, where $\epsilon_{total}=n_g\epsilon_{{\emph{g}}}$, and $n_g$ denotes the total iterations of generator in this algorithm.
\end{theorem}
\begin{remark}
For the calculation of the total privacy budget, it cannot directly using the composition property (Lemma~\ref{33}) as there is a strict condition in Theorem~\ref{44}, i.e., $\log \left( {1/\delta } \right) \ge {\epsilon ^2}n_d$. For typical DP setting, we can choose $\delta=10^{-5}$ and $\epsilon=1$, then we can obtain $n_g\leq 11.5$. Thus this condition cannot be satisfied in GAN scenarios as we always set a large number of iterations, usually $n_g>1000$. So we directly sum all the privacy budgets of every generator iteration in Theorem~\ref{total}.
\end{remark}

\section{Adaptive Noise Tuning Algorithm}
Although the loss value is bounded naturally, adding random noise may reduce the learning performance. In Algorithm~\ref{gan1}, the privacy budget allocated to one iteration is evenly distributed, which implies a same noise scale of Gaussian mechanism in every generator iteration. In general, the final model performance, such as convergence, accuracy, etc., is largely dependent on the amount of noise added over the training process \cite{yu2019differentially}. Thus, in order to obtain a better learning performance, we can optimize the allocation of privacy budget and design the amount of noise added in each iteration. The details are described as follows.
\subsection{Adaptive Noise Tuning Algorithm}
Our adaptive noise tuning algorithm follows the idea that, as the training continues, it is expected to have less noise on the loss function, which allows the model converge faster and obtain better results. Similar ideas have been applied by adjusting the learning rate in common practice \cite{chollet2017xception}. Therefore, we propose an algorithm that can adaptively tune the noise scale over the training iterations, and effectively improve the model accuracy in turn.

Our approach for adaptive tuning the noise scale is to monitor the training accuracy and reduce the noise scale when the accuracy stops improving. To do this, we first add an extra testing process after each generator finishes training in each iteration. Every time when the improvement of the testing accuracy is less than a predefined threshold $\tau$, the noise scale is reduced by a factor of $k$ until the total privacy budget runs out. Although this will lead to recalculating privacy cost, we can show the improvement on the performance is convincing despite more energy cost in the experimental results.

In our approach, with a validation dataset, the testing accuracy is checked periodically during the training process in the generator to determine whether the noise scale needs to be reduced for subsequent iterations. Let $\sigma_t$ be the noise scale in the $t$th iteration obtained by Theorem~\ref{total} when the total privacy budget and total iterations are given, and $S_t$ be the corresponding testing accuracy. The noise scale for the subsequent iterations is adjusted based on the accuracy difference between the current and previous iteration.
\begin{equation} \label{test}
\sigma_{t+1}=\left\{
  \begin{array}{ll}
    \sigma_t, & \hbox{if $S_{t+1}-S_{t}\geq \tau$;} \\
    k\sigma_t, & \hbox{else,}
  \end{array}
\right.
\end{equation}
where $\tau$ denotes the threshold for the testing accuracy.

Initially we set $0<k\leq1$ and $S_0=0$.
Then the updated noise scale $\sigma_{t+1}$ will be applied to the next iteration.

In addition, we find that the testing accuracy may not increase monotonically as the training progresses, and its fluctuation may cause unnecessary reduction of noise scale and thus wasting  the privacy budget. This motivates us to use the average of testing accuracy to improve the algorithm: at the current testing iteration $t$, we define an average testing accuracy $\overline{S}_t$ over the previous iterations as follows:
\begin{equation}
{\overline S_t} = \frac{1}{t}\sum\limits_{i = 1}^t {{S_i}}.
\end{equation}
The average testing value will replace the previous one in Eq.(\ref{test}) and determine whether the noise scale should be reduced or not.

To sum up, we formally present our adaptive noise tuning DP-GAN framework in the Algorithm~\ref{gan2}.

\begin{algorithm}[htb]
  \caption{ Framework of the adaptive noise tuning RDP-GAN.}
  \label{gan2}
  \begin{algorithmic}[1]
    \Require
      real samples: $\{x_1,x_2, \cdots\}\sim p(x)$;
      prior samples: $\{z_1,z_2, \cdots\}\sim p(z)$;
      sampling rate: $q$;
      number of discriminator iterations per generator iteration: $n_d$;
      number of generator iteration: $n_g$;
      decay rate: $k$;
      total privacy budget: $\epsilon_{total}$;
    \Ensure
      differential private generator.
    \State Initialize discriminator parameters $d_0$, generator parameters $g_0$.
    \State According to the total privacy budget $\epsilon_{total}$ and the total number of generator iterations $n_g$, estimate the noise scale for each generator iteration by Theorem~\ref{11} and Theorem~\ref{total}. The algorithm ends when $\epsilon_{total}=0$.
    \State According to the total number of discriminator iterations $n_d$, estimate the initial noise scale for each generator iteration by Theorem~\ref{22}.
    \State \textbf{For} $t_1 =1, . . ., n_g$ do

    \State ~~\textbf{For} $t_2 = 1, . . ., n_d$ do

    \State ~~~~Sample $\{x_1,x_2, \cdots\}\sim p(x)$ a batch of $m$ real data points.

    \State ~~~~Sample $\{z_1,z_2, \cdots\}\sim p(z)$ a batch of $m$ prior samples.

    \State ~~~~Train the discriminator $d$.

    \State ~For each data sample $i$, $F_d(W;X)\leftarrow \sum ^ {m}_{i=1}F_d^{t_2}(w_i,x_i)+\mathcal{N}(0,\sigma_t^2)$.

    \State ~~~~Update the discriminator $d$ according to Eq.~(\ref{u1}) and (\ref{u2}).


    \State ~~\textbf{End for}

    \State ~~Update the generator $g$ and generates fakes samples to test.
    \State ~~According to the Eq.~(\ref{test}), update the noise scale $\sigma_{t_1}$ to $\sigma_{t_1+1}$ and calculate the rest $\epsilon_{total}$.
    \State ~~Train the discriminator $d$.
    \State \textbf{End for}

    \State Return differentially private generator $g$.
  \end{algorithmic}
\end{algorithm}
\subsection{Pre-defined Noise Decay Schedules}
In this subsection, we will list several predefined noise decay schedules, and provide comprehensive comparison results with our proposed algorithm in Sec.~5.5.
\subsubsection{Time-Based Decay} According to \cite{darken1991note}, it is defined with the mathematical form
\begin{equation}
\sigma_t=\sigma_0/\left(1+kt\right),
\end{equation}
where $\sigma_0$ is the initial noise scale, $k$ is the decay rate and $t$ is the current iteration, respectively.

\subsubsection{Exponential Decay} It has the mathematical form that
\begin{equation}
\sigma_t=\sigma_0*e^{-kt}.
\end{equation}
\subsubsection{Step Decay} Step decay reduces the learning rate by some factor every few iterations. The mathematical form is expressed as
\begin{equation}
\sigma_t=\sigma_0*k^{t/t^{*}},
\end{equation}
where $t<t^*$ decides how often to reduce noise in terms of the number of iterations.
\subsection{Privacy Preserving Parameter Selection}
The proposed algorithms require a set of pre-defined hyperparameters, such as decay rate $k$. This value will definitely the training time and affect the final model accuracy. Intuitively speaking, a smaller decay rate $k$ will lead to a better performance but longer training iterations. It is expected to find an appropriate decay rate for the schedules under an optimal tradeoff between the learning performance and training. However, this decay rate should be determined by multiple factors such as typical neuron network, training set and other factors. Thus in this subsection, we only propose a straightforward method to obtain an appropriate value. To list $k$ candidates by training $k$ neural networks respectively and directly choose the one that achieves the highest accuracy, accordingly we show the experimental results in Sec.~V-E.
\section{Experimental Results}
In this section, we first evaluate the performance of our analytical results for different privacy budgets. Then, we evaluate our proposed RDP-GAN method on the tabular and image datasets, respectively. In addition, we demonstrate the effectiveness of the improved algorithm with various noise scaling methods and parameter settings.

\subsection{Experimental Setting}
\subsubsection{Dataset}
In our experiments, we use two benchmark datasets for different tasks:
\begin{itemize}
  \item Adult, which consists of $30K$ individual information, has $14$ attributes ($8$ selected in the experiments) for each individual and split into $20K$ training and $10K$ test samples. 
      We show an Adult dataset sample with some attributes in Table~\ref{tab1}.
      \begin{table}
\centering
\caption{Some examples of the Adult dataset}\label{tab1}
\resizebox{80mm}{10mm}{
\begin{tabular} {|c|c|c|c|c|c|c|c|}
  \hline
  Age  &Occupation& Education &Gender& Work Class& Marital Status & Work Hour & Income \\
  \hline
  39  &Sales& Bachelors & Male  & State-gov & Never-married& 30&$\leq 50K$ \\
  \hline
  38  & Tech-support& HS-grad & Male & Private& Divorced & 35& $\leq 50K$\\
  \hline
  44 & Prof-specialty &Masters  & Male & Private& Married-civ-spouse&42& $> 50K$ \\
  \hline
  43 & Other-service & Some-college & Female & State-gov& Separated&26&$\leq 50K$ \\
  \hline
\end{tabular}}
\end{table}
  \item MNIST, which consists of $70K$ handwritten digit images of size $28 \times 28$, split into $60K$ training and $10K$ test samples. 
      We show the typical samples of the standard MNIST dataset in Fig.~\ref{mnist}.
\begin{figure}
\centering
  \includegraphics[width=0.4\textwidth]{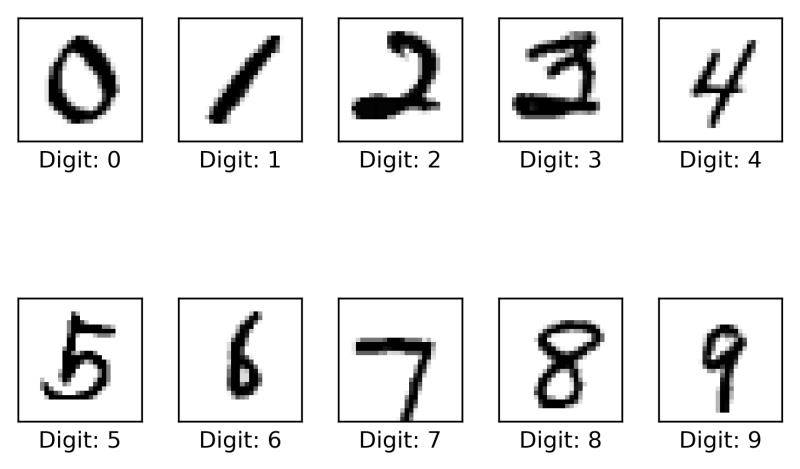}
  \caption{Some samples of the MNIST dataset } \label{mnist}
\end{figure}
\end{itemize}

\subsubsection{GAN}
The neural network architecture is adapted to the deep convolutional GAN (DCGAN), where the discriminator is a convolutional neural network (CNN) that contains $3$ layers. In each layer, a list of learnable filters are applied to the entire input matrix, where data samples are converted into square matrices in our method. To make the loss value bounded, we choose the Sigmoid activation function as the last layer that generates the probability of being real or fake. The loss function of the discriminator is set to use Cross Entropy function. The generator is also a neural network that consists of $3$ de-convolutional layers.

In the experiment, we set the learning rate of discriminator $a_d$ and generator $a_g$ to be $5.0 \times 10^{-5}$. The sampling rate $q$ is set according to different batch sizes, and the number of iterations on discriminator ($n_d$) and generator ($n_g$) are $5$ and $10^3$, respectively.

\subsubsection{Privacy and Utility Level Estimation}
In the experimental results, we set up two privacy levels: $\epsilon_{total}=0.5$ and $\epsilon_{total}=5$, respectively. To verify the generated results, we use an additional classifier to test the accuracy with right labels in the MNIST dataset. The classifier is trained by the $10K$ test samples in the MNIST dataset and is used to distinguish whether an input digit with a predefined label is belong to the right label. Moreover, the classifier uses a multilayer perceptron (MLP) neural network with $3$ layers. Except for the input nodes, each node is a neuron that uses Sigmoid activation function.
In the Adult dataset,
we conducted the probability mass function (PMF) and the absolute average error with the true data for each attribute. For the relationship among different attributes, we use a trained classifier to test the accuracy. In addition, the absolute average error is calculated by the following equation:
\begin{equation}
\textrm{Error} = \sum\limits_{x = 1}^{\left| X \right|} {\left| {{p_g}(x) - {p_r}(x)} \right| \times x};
\end{equation}
where $|X|$ denotes the total number of labels, $p_g(x)$ and $p_r(x)$ denote the PMF of the generated and real data, respectively.

\subsection{Analytical Results on the Privacy Level}
\subsubsection{Investigation on the Sensitivity}
Before we show the analytical results of the proposed RDP-GAN, we need to investigate the sensitivity on the loss function. According to Eq.~(\ref{eq:sen}), we can first estimate the bound of loss function and calculate the sensitivity. We conduct experimental results on the Adult and MNIST dataset, respectively, and record the maximum loss value without injecting noises in In Table~\ref{tab2}.
\begin{table}
\centering
\caption{Simulation results for the sensitivity value on the loss function.} \label{tab2}
\resizebox{80mm}{10mm}{
\begin{tabular}{|c|c|c|}
  \hline
  {} & Batch size $= 64$ & Batch size $= 128$ \\
  \hline
  Maximum value of loss function on Adult & 17.852 & 21. 242\\
  \hline
  Maximum value of loss function on MNIST & 23.534 & 26.823 \\
  \hline
  $C$ on Adult & 20 & 23\\
  \hline
  $C$ on MNIST & 25 & 28 \\
  \hline
  $\Delta S$ on the Adult dataset & 0.3125 &  0.1796 \\
  \hline
  $\Delta S$ on the MNIST dataset & 0.1953 & 0.2187 \\
  \hline
\end{tabular}}
\end{table}
From the results we can find a large batch size will lead to a large maximum loss function value on both datasets. Accordingly, we set $C$ for different situations in Table.~\ref{tab2}\footnote{Note here the value of $C$ is only used to estimate the privacy loss. Nevertheless, we will show the proposed privacy analysis can achieve a lower bound compared to other method under same noise.}. The results also indicate that one individual will produce less influence in a larger dataset.

\subsubsection{Analytical Results on the Privacy Level of Generator}
To show the effectiveness of the RDP-GAN, we compare it with the method using Moment Accountant (MA) algorithm \cite{8636556,abadi2016deep} as shown in the follow figures. In \cite{8636556,abadi2016deep}, the privacy level of generator $\epsilon_g$ in each iteration can be expressed as
\begin{equation}
{\epsilon_g} = \frac{{2q\sqrt {{n_d}\log \left( {\frac{1}{\delta }} \right)} }}{\sigma },
\end{equation}
where the sensitivity is omitted in the formula. In addition, the expression of $\epsilon_g$ in MA has a order of $O(\frac{1}{\sigma})$, where our analysis method (RDP) will achieve a order of $O(\frac{1}{\sigma^2})$ respect of the injected noise. This will lead to a higher privacy level in RDP when the noise scale is qualitative larger, e.g., $\sigma > 10$.
\begin{figure} \label{fig_privacy}
\centering
\subfigure[Batch size $= 64$]{\label{fig_oneway}
  \includegraphics[width=0.4\textwidth]{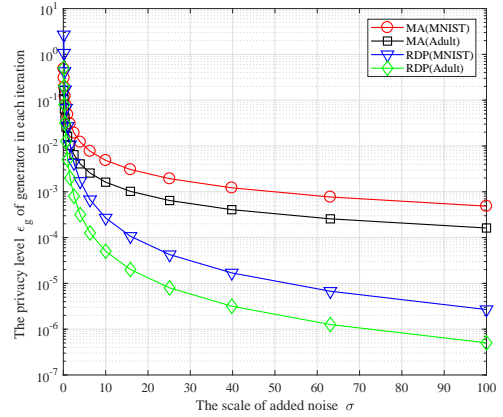}
}
\subfigure[Batch size $= 128$.]{\label{fig_twoway}
  \includegraphics[width=0.4\textwidth]{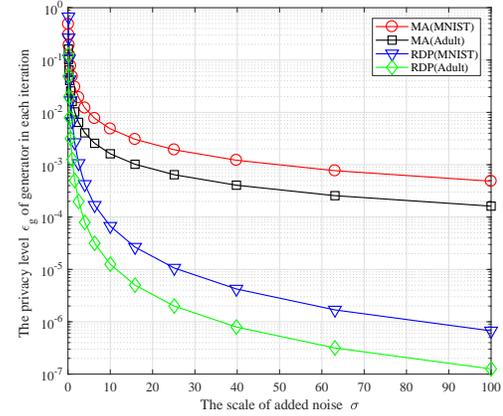}}
 \caption{The analysis privacy level with different noise scales.} \label{fig_privacy}
\end{figure}

In Fig.~\ref{fig_privacy} we investigate the privacy level in two batch sizes. As can be found in these figures, the proposed RDP method can achieve a better privacy level when larger noises are added as expected. Moreover, the privacy level in the Adult dataset is always higher (a smaller $\epsilon$) than the MNIST dataset. The intuition behind this is that the sensitivity in the Adult dataset is lower than it in the MNIST dataset as it has less attributes.

\subsection{Experimental Results on the MNIST Dataset}
In this subsection, we excute the proposed algorithm on the MNIST dataset. In Fig.~\ref{fig_msample}, we first show the real and generated data samples with different noise scale.
We also show the results generated by the DP-GAN \cite{8636556,xie2018differentially} for comparison, which adds noise on the gradients. From the figure we can find the generated digits with low privacy level can well preserve the characters of the true ones, and with the increase of the privacy level, the quality begins to drop as there are unclear pixels exists. For example, we can find the digit $2$ and digit $8$ are hard to recognize in the generated samples when $\epsilon_{total}=5$. Moreover, the quality generated by DP-GAN is obviously worse than the proposed RDP in both privacy levels.
\begin{figure}
\centering
  \includegraphics[width=0.4\textwidth]{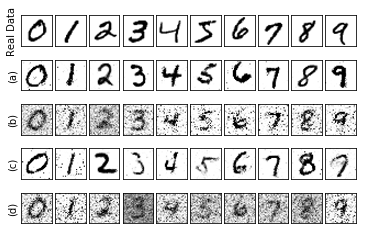}
  \caption{Generated MNIST samples with different noise scale: (a) RDP ($\epsilon_{total}=0.5$); (b) DP-GAN ($\epsilon_{total}=0.5$); (c) RDP ($\epsilon_{total}=5$); (d) DP-GAN ($\epsilon_{total}=5$).} \label{fig_msample}
\end{figure}

In order to further verify the qualities of the generated samples, we use an additional classifier to test the accuracy.
We show the test accuracy along with the training process (i.e., $1000$ iterations) in the following figures.
\begin{figure} \label{macc}
\centering
\subfigure[No noise]{\label{0}
  \includegraphics[width=0.22\textwidth]{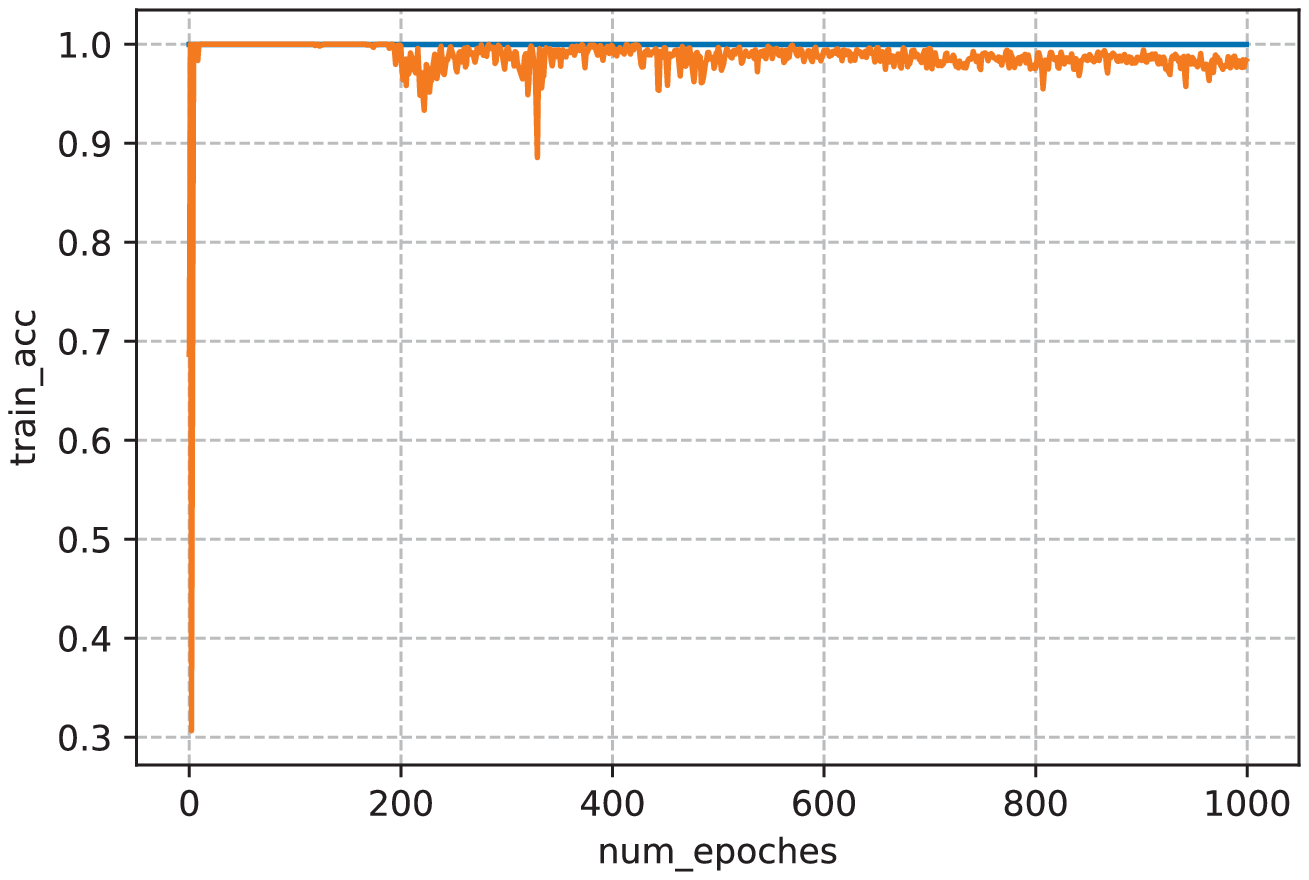}
}
\subfigure[$\sigma=0.5$]{\label{0.5}
  \includegraphics[width=0.22\textwidth]{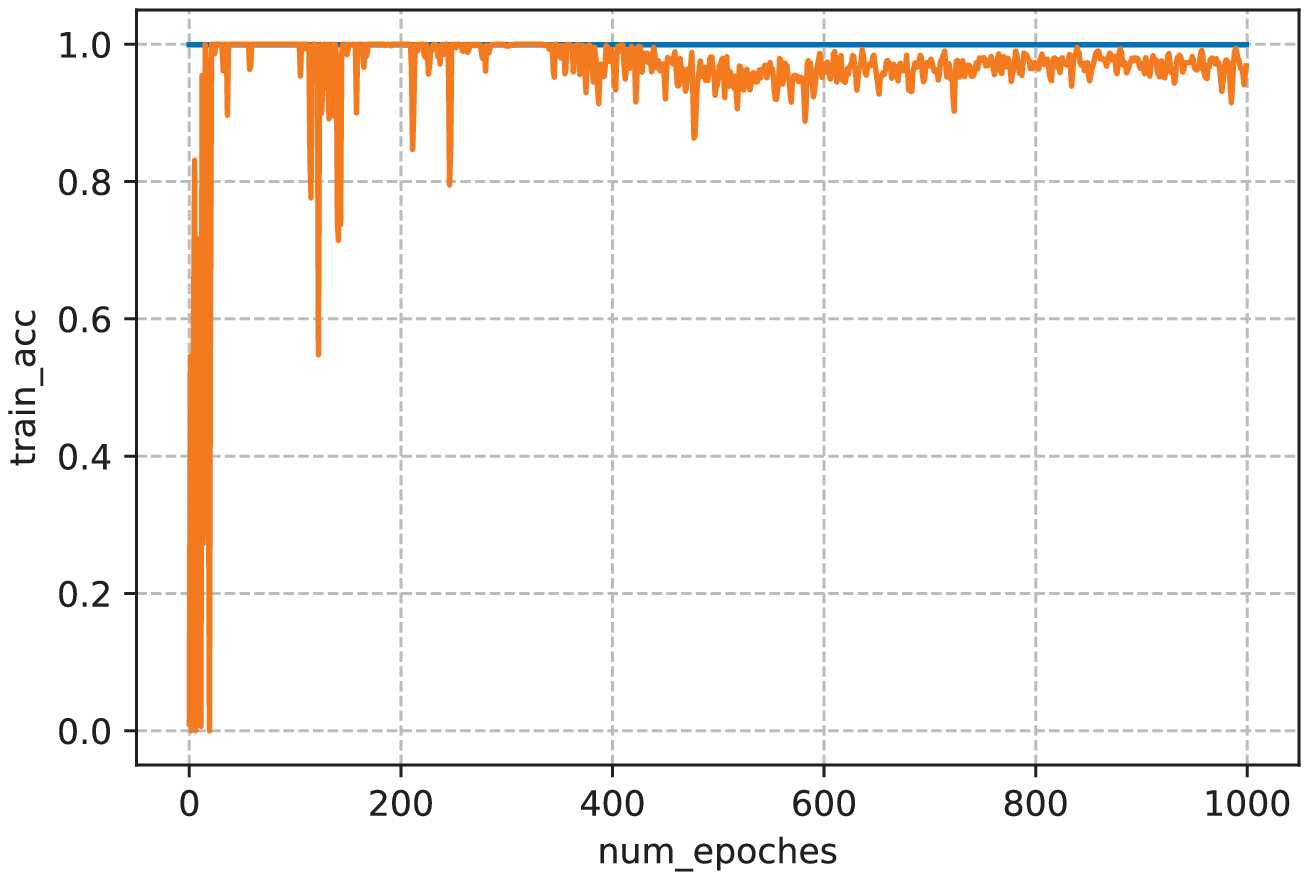}}
\subfigure[$\sigma=5$]{\label{5}
  \includegraphics[width=0.22\textwidth]{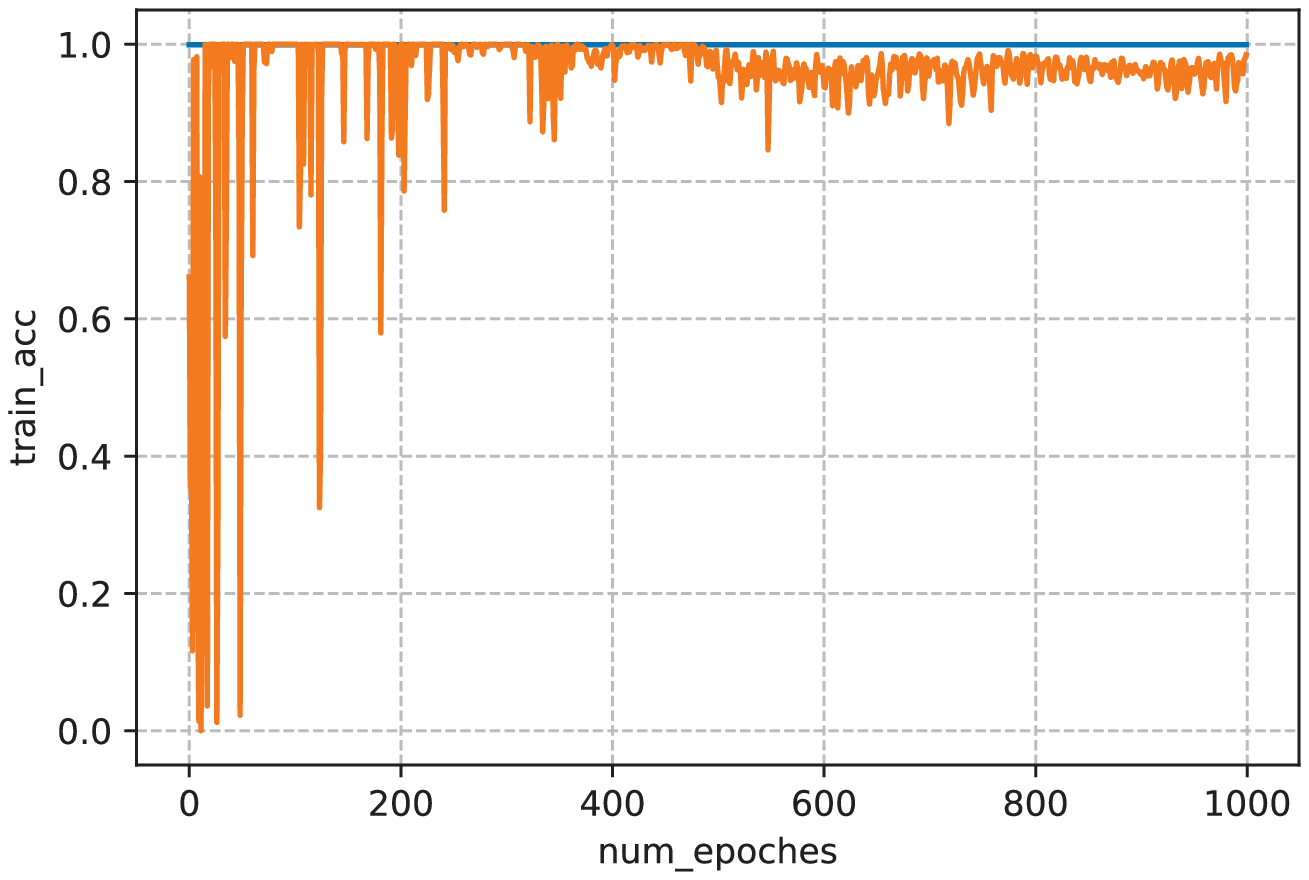}}
\subfigure[$\sigma=10$]{\label{10}
  \includegraphics[width=0.22\textwidth]{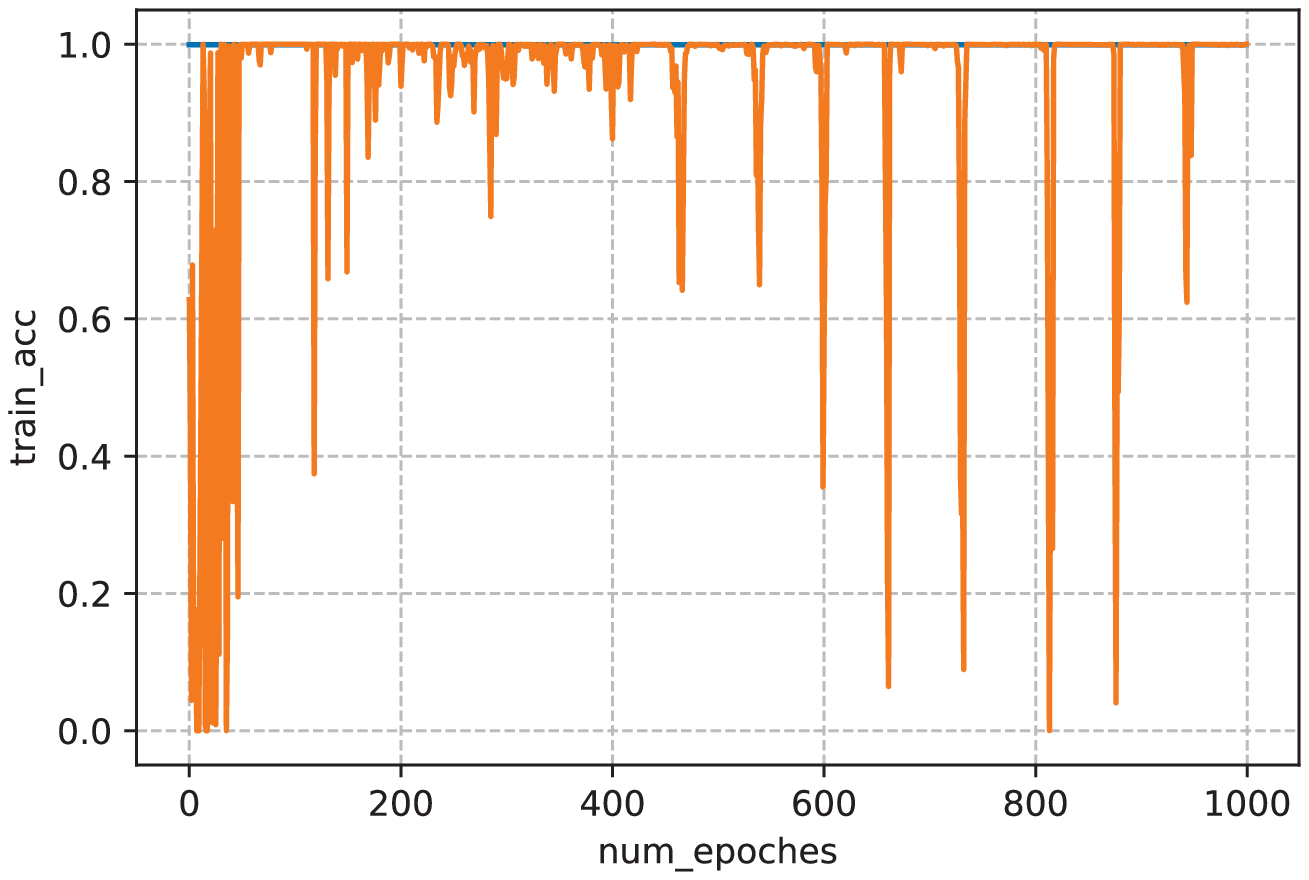}}
 \caption{The classification accuracy with different noise sacles in MNIST dataset.} \label{macc}
\end{figure}
From Fig.~\ref{macc} we can first find that the generated samples without privacy protection, i.e., $\sigma=0$, can achieve relative high accuracy in the classifying. Second, with the increase of the noise scale, the accuracy performance begins to fluctuate, which is brought from the injected noises. In addition, a few collapse cases happen when a large noise is added. For example, we can find in Fig.~\ref{10} the accuracy always suddenly drops to $0.3$ or less. This is because the large scale of $\sigma$ will probably brings about huge noises to the loss function that the system cannot stand. Therefore, it will lead to unacceptable results in the generated samples, i.e., the digit $8$ in Fig.~\ref{fig_msample} when $\sigma=20$.

\begin{figure} \label{fig_MC}
\centering
\subfigure[$\epsilon_{total}=0.5$]{\label{fig_M05}
  \includegraphics[width=0.4\textwidth]{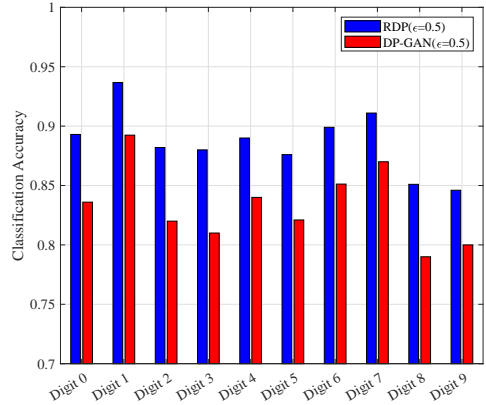}
}
\subfigure[$\epsilon_{total}=5$]{\label{fig_M5}
  \includegraphics[width=0.4\textwidth]{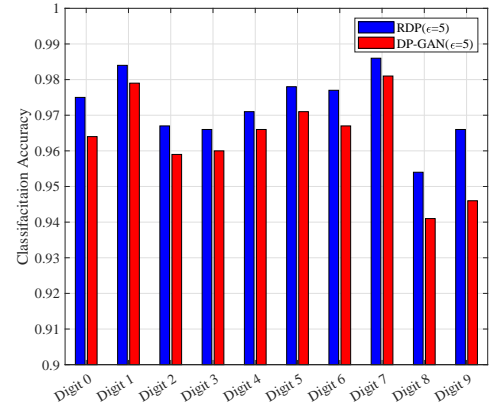}}
 \caption{The classification accuracy comparison with DP-GAN.} \label{fig_MC}
\end{figure}

In addition, we compare our algorithm (RDP) with the method that adds noise on the parameters, i.e., DP-GAN in two privacy levels, $\epsilon_{total}=0.5$ and $\epsilon_{total}=5$, respectively. As can be found in Fig.~\ref{fig_MC}, the proposed algorithm has a better performance compared with DP-GAN in both privacy levels, in which the biggest performance gap is around $7\%$ when classifying digit $3$ in the high privacy level. As explained in Sec. 3.1, the superior performance of the proposed RDP-GAN compared with the DP-GAN is because that adding noise on the loss function can bring in more explicit privacy protection.
\subsection{Experimental Results on the Adult Dataset}
In this subsection, we show the performance of the proposed algorithm on the Adult dataset.
\begin{figure} \label{adult}
\centering
\subfigure[Age]{\label{age}
  \includegraphics[width=0.22\textwidth]{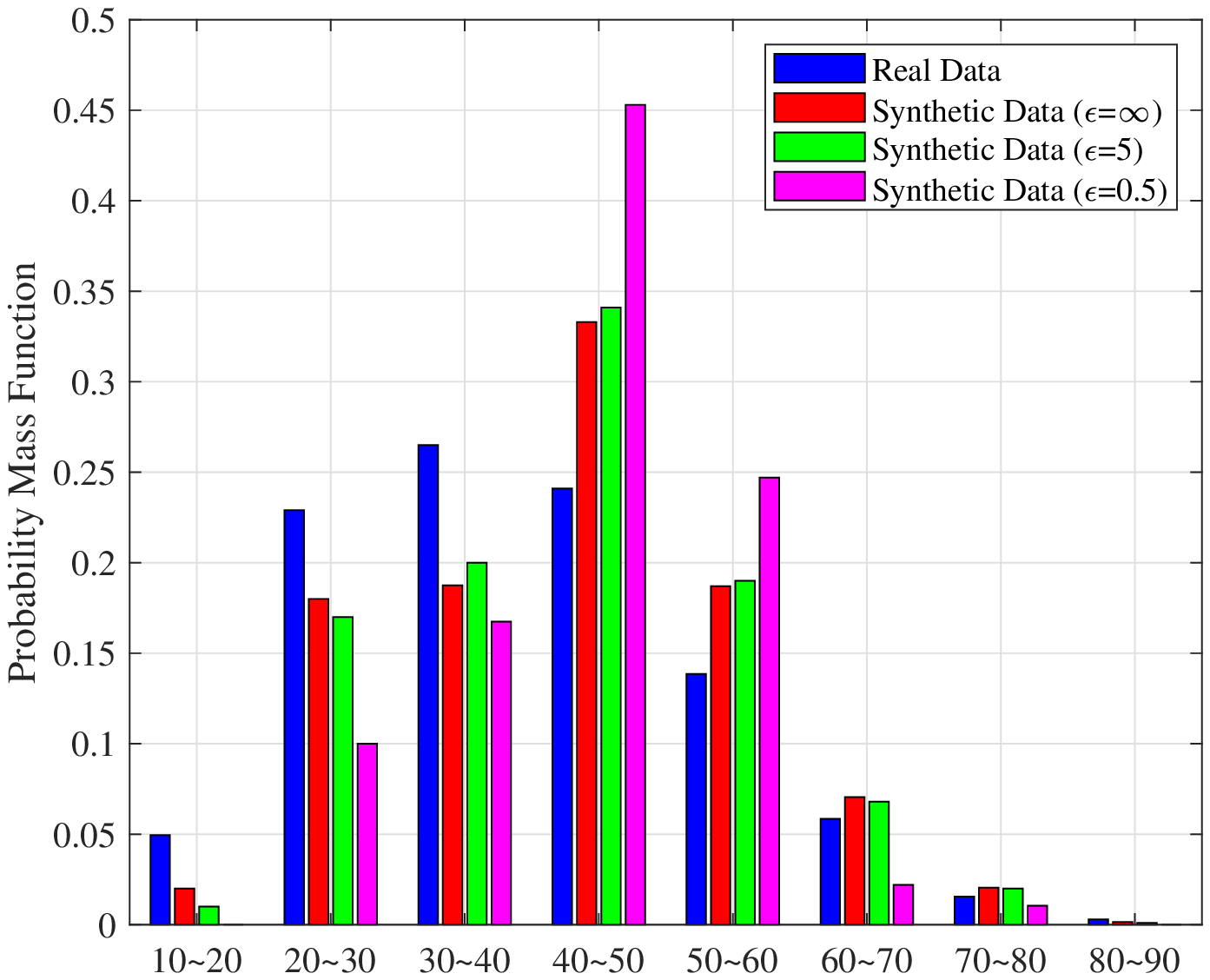}
}
\subfigure[Occupation]{\label{occ}
  \includegraphics[width=0.22\textwidth]{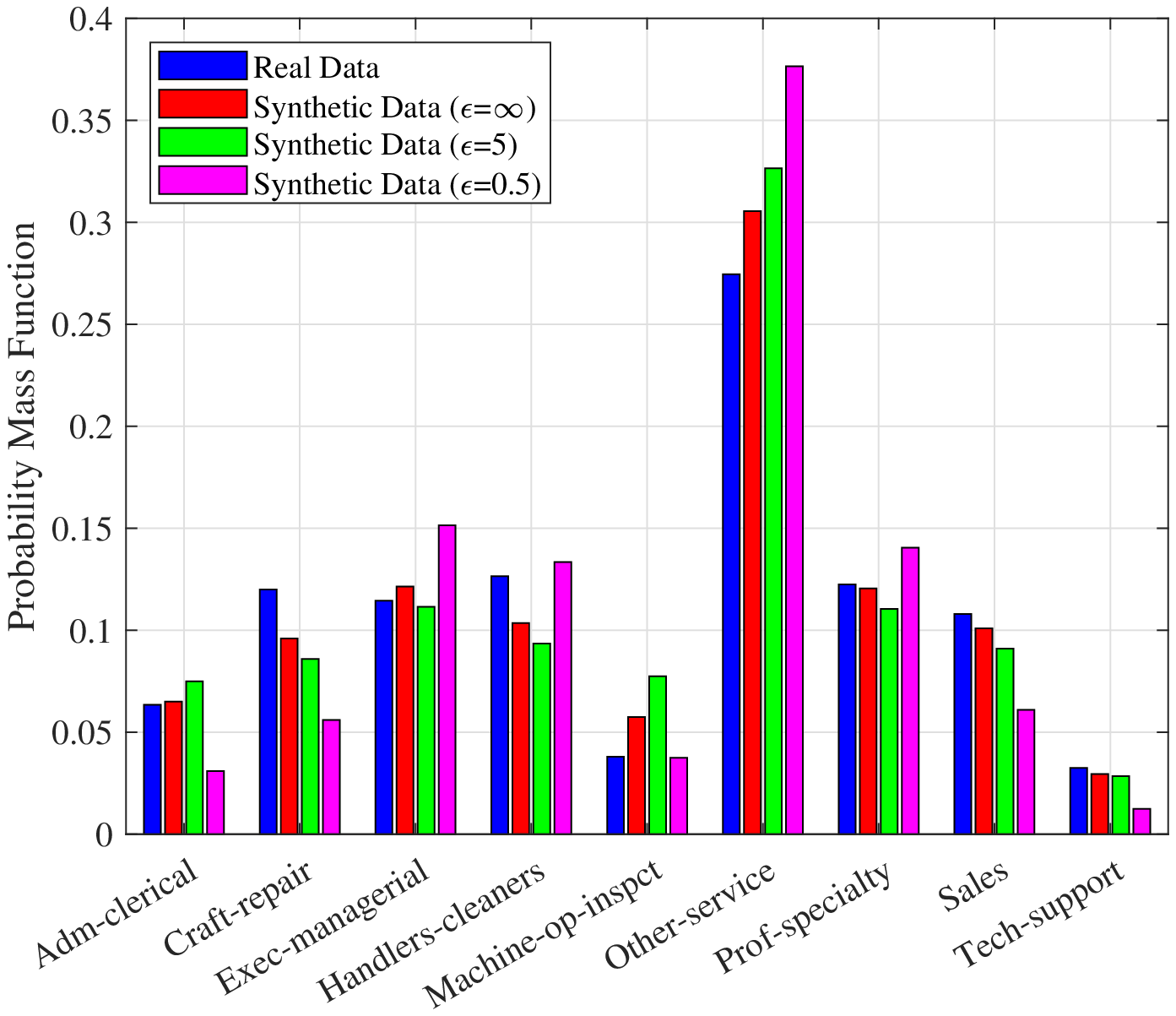}}
\subfigure[Education]{\label{edu}
  \includegraphics[width=0.22\textwidth]{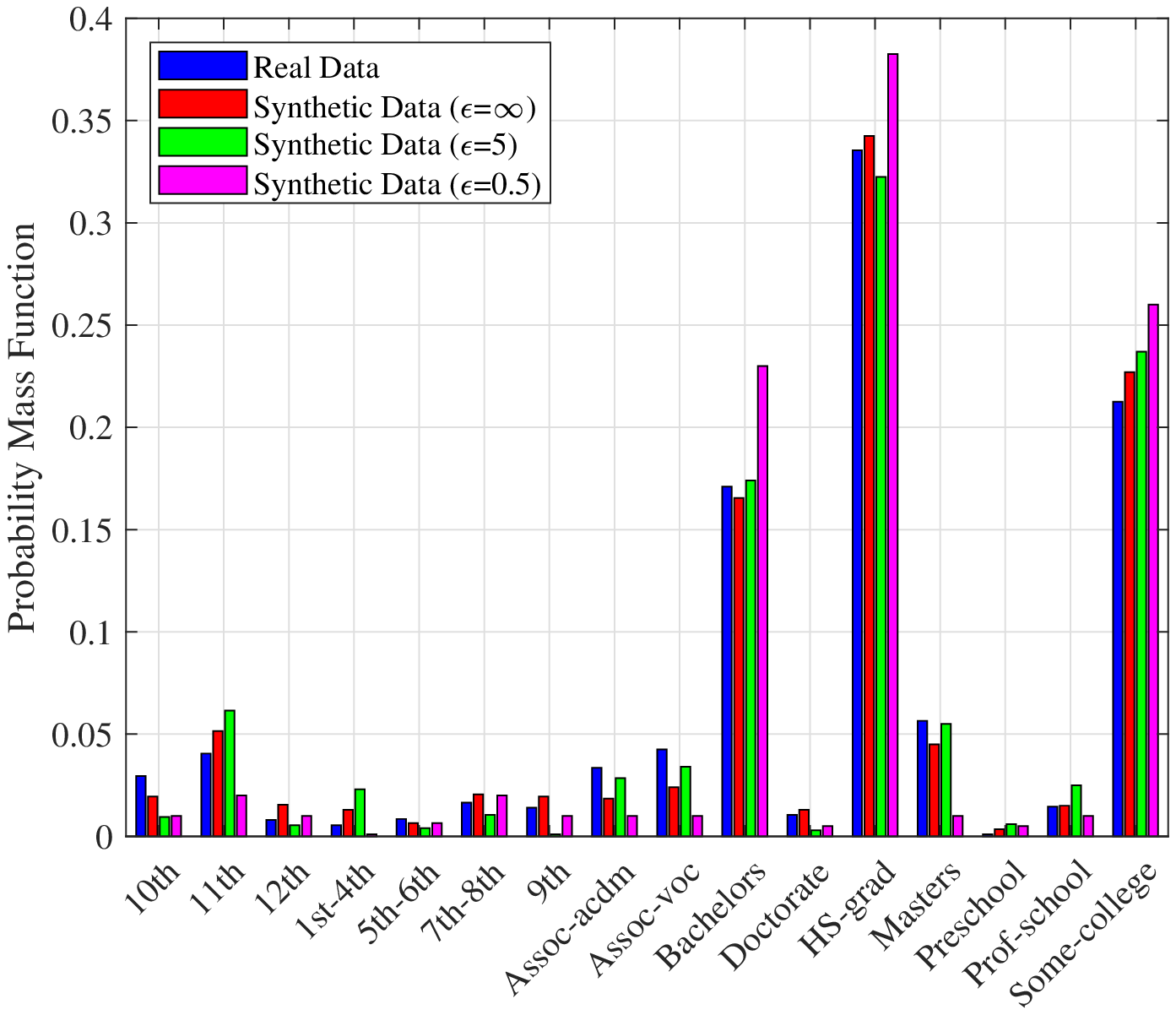}}
\subfigure[Gender]{\label{gen}
  \includegraphics[width=0.22\textwidth]{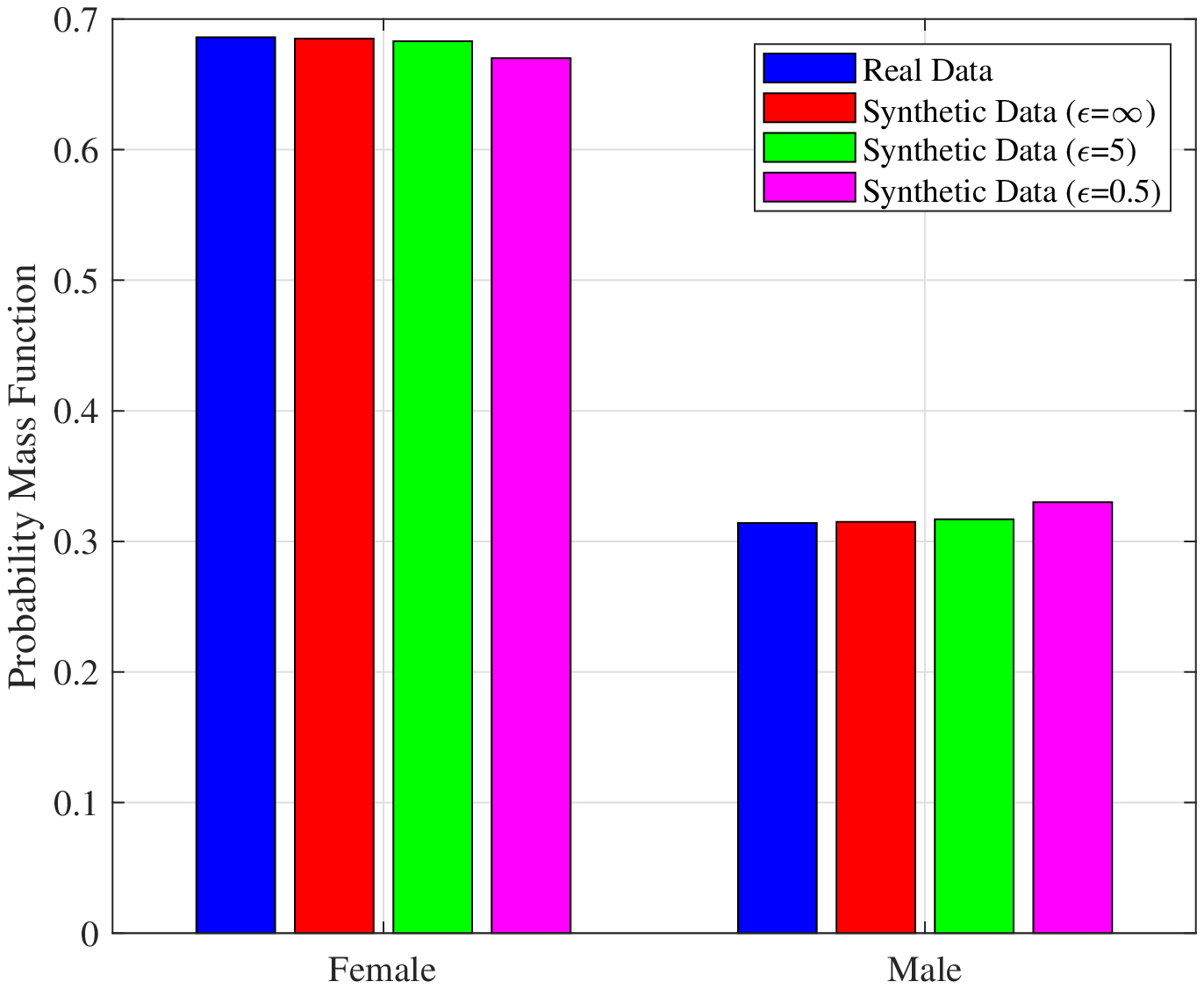}}
  \subfigure[Work Class]{\label{gov}
  \includegraphics[width=0.22\textwidth]{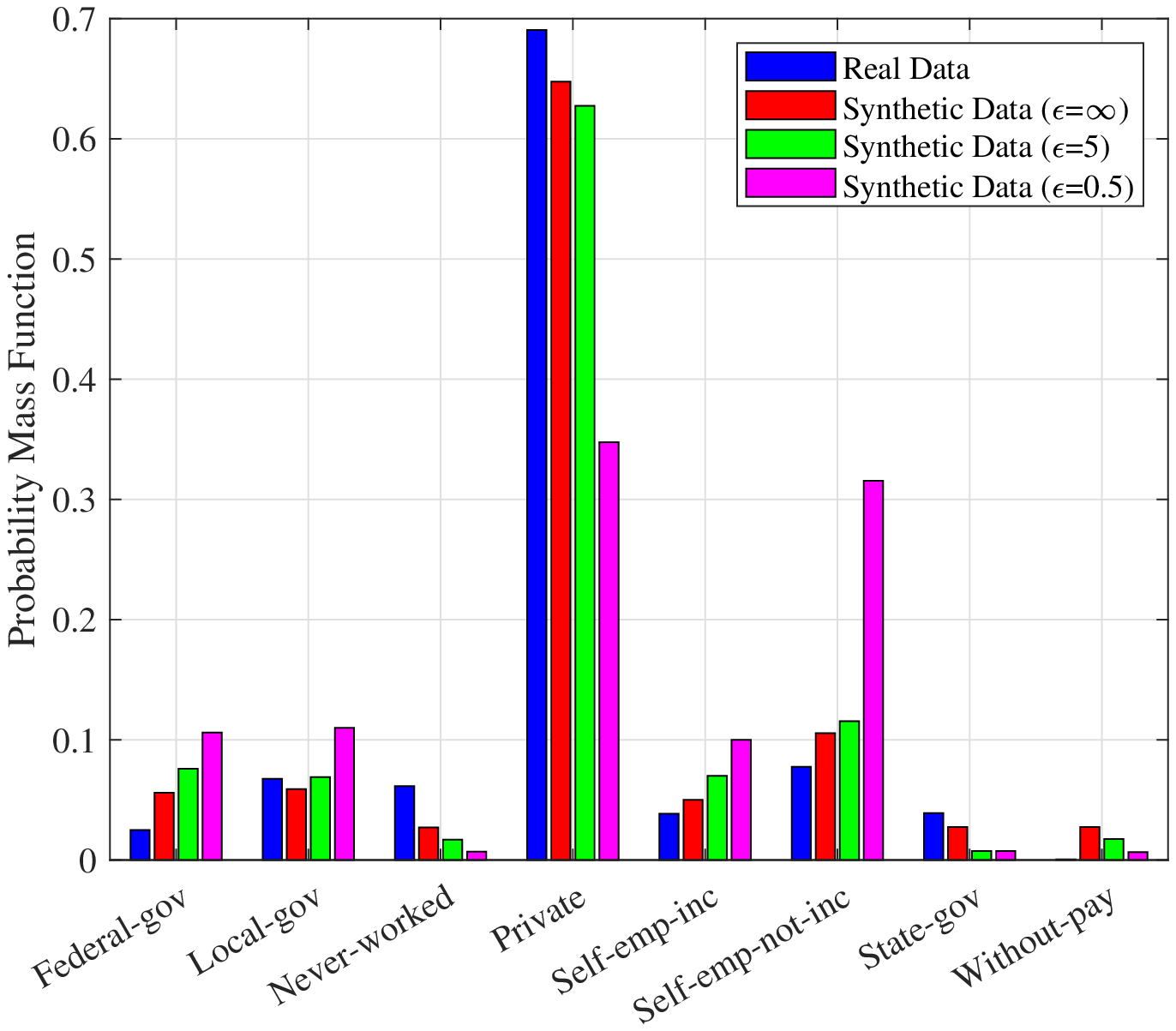}}
   \subfigure[Marital Status]{\label{mar}
  \includegraphics[width=0.22\textwidth]{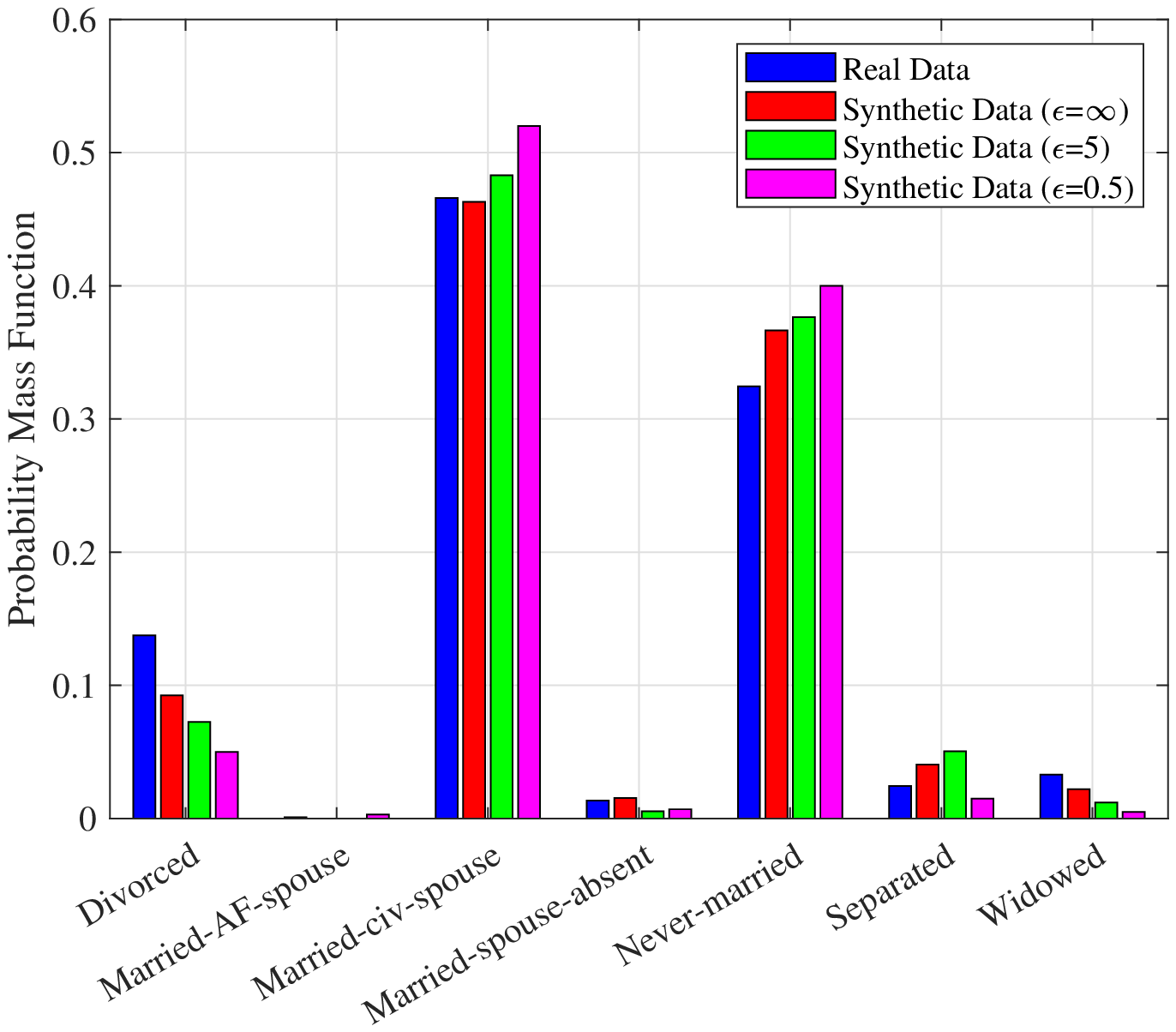}}
  \subfigure[Work Hour]{\label{work}
  \includegraphics[width=0.22\textwidth]{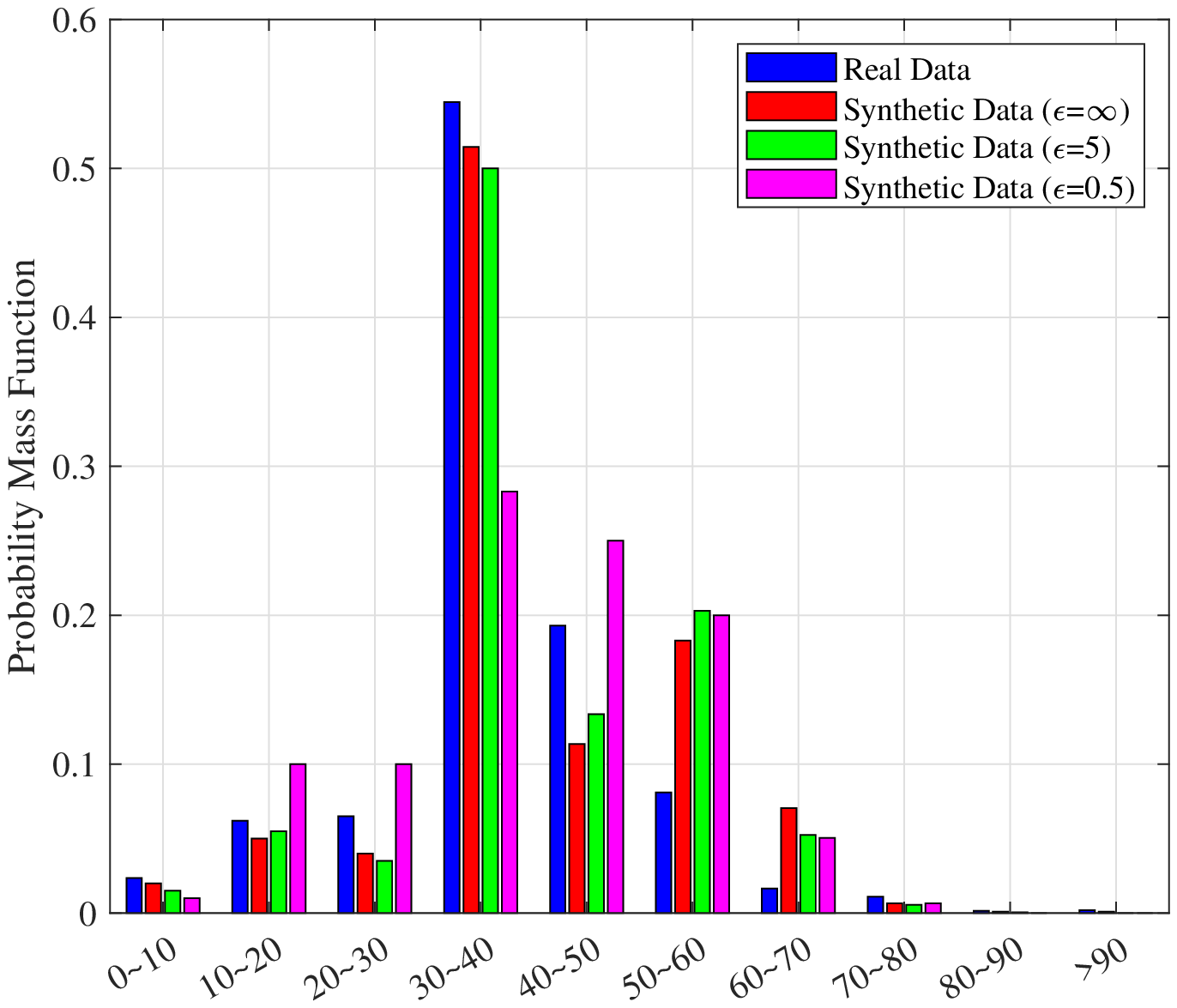}}
  \subfigure[Income]{\label{inc}
  \includegraphics[width=0.22\textwidth]{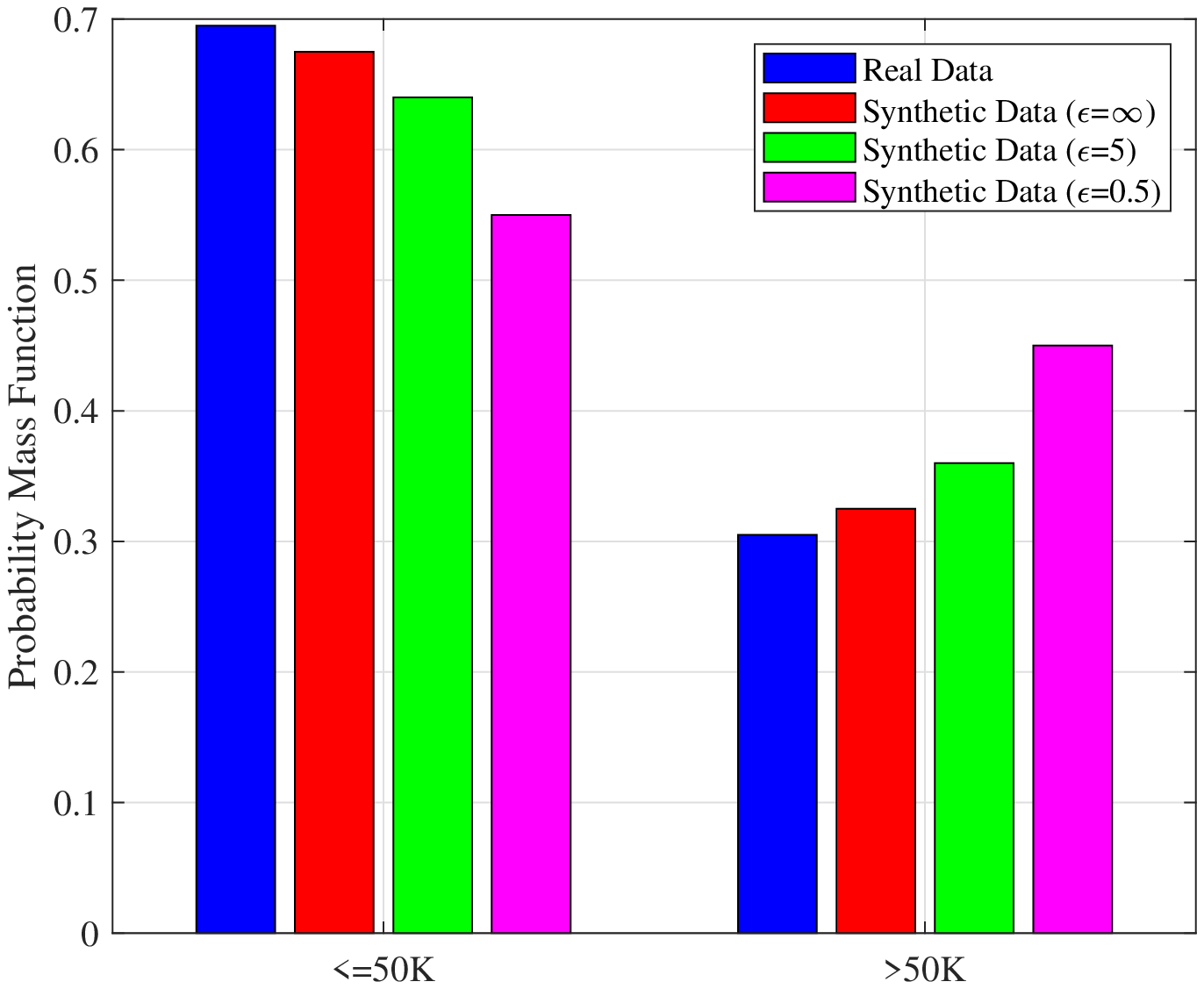}}
 \caption{The PMF of generated {synthetic} samples.} \label{adult}
\end{figure}
In Fig.~\ref{adult} we show the probability mass function (PMF) of the generated attributes. Take the attribute of Age as an example, we use binning to divide the age range into $9$ regions as shown in Fig.~\ref{age} to increase the training speed,. As can be found in Fig.~\ref{age}, the generated samples can achieve similar distribution with the true ones, while the samples with larger noise make more differences. Moveover, the largest difference is focused on the edge domain. The main reason may because the generator always produces samples that are likely to appear in the true dataset, i.e., the one close the average value. Similar phenomenon can be found in other attributes as well. To further investigate this the quality of the generated samples, we show the absolute average error for these attributes. As can be found in Table.~\ref{error}, when more noises are injected, the absolute average error increases. Thus, the proposed algorithm may not have the capability to achieve high performance in maintaining the character of unique attributes, especially in the tabular dataset. To further enhance this performance, it may need more intelligent design on the generator, which is left as our future work.
\begin{table}
\centering
\caption{Absolute average error for the generated samples with different privacy levels.} \label{error}
\resizebox{80mm}{6mm}{
\begin{tabular}{|c|c|c|c|c|c|c|c|c|}
  \hline
  {} & Age & Occ. & Edu.& Gen.& Work C&Mar.&Work H&Inc.\\
  \hline
  $\epsilon=\infty$ & 1.09& 0.68 & 0.92& 0.01& 0.834& 0.36&1.66&0.06\\
  \hline
  $\epsilon=5$ & 1.1 & 1.23 & 1.23& 0.009& 1.11& 0.57&1.65&0.17\\
  \hline
  $\epsilon=0.5$ & 3. & 1.75 & 1.87& 0.048& 2.21 & 0.69&2.55&0.44\\
  \hline
\end{tabular}}
\end{table}

To verify the correlation among different attributes, we use a trained classifier to determine whether the income of an individual is larger than $50K$ or not. Experimental results on the comparison of the proposed RDP-GAN with DP-GAN can be found in Fig.~\ref{testacc}.
\begin{figure}
\centering
\subfigure[Classifier trained by real data]{\label{bar1}
  \includegraphics[width=0.4\textwidth]{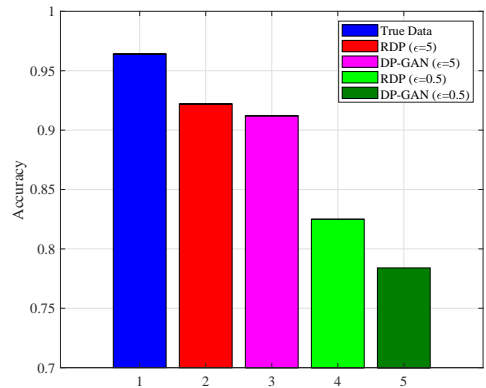}
}
\subfigure[Classifier trained by synthetic data]{\label{bar2}
  \includegraphics[width=0.4\textwidth]{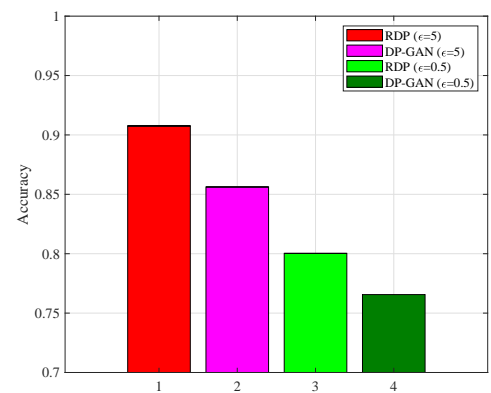}}
  \caption{Correlation comparison with DP-GAN.} \label{testacc}
\end{figure}
From Fig.~\ref{bar1} we can find the classifier is able to obtain a fairly high accuracy trained by the true data samples, which has around $96.4\%$ accuracy. Moreover, compared with DP-GAN, the proposed RDP algorithm can better capture the relationship among attributes in both privacy level. For example, when $\epsilon=0.5$, the performance gap is around $4.1\%$ ($82.5$ v.s. $78.4\%$). In addition, in Fig.~\ref{bar2} we trained the classifier using the synthetic data, and test its accuracy using the true data. The results also show that the proposed algorithm can achieve a better performance than DP-GAN at a same privacy level.
\subsection{Experimental Results on the Adaptive Noise Tuning Algorithm}
In this subsection, we first conduct experiments to find the appropriate decay rate for different algorithms. Specifically, we choose $\epsilon_{total}=0.5$ and $1000$ iterations, which leads to adding $\sigma=10$ to the discriminator in each iteration, and the testing accuracy $\tau$ is set to $0.5$ to prevent model collapsing. In addition, we denote the Time-based decay, Exponential decay and Step decay as Time, Exp and Step, respectively, and the $t^*$ in Step is set to $100$. In the following, we show the trend of the testing accuracy with different value of decay rate $k$ in the Adult dataset.
\begin{figure} \label{scale}
\centering
\subfigure[Time and Exp]{\label{scale1}
  \includegraphics[width=0.22\textwidth]{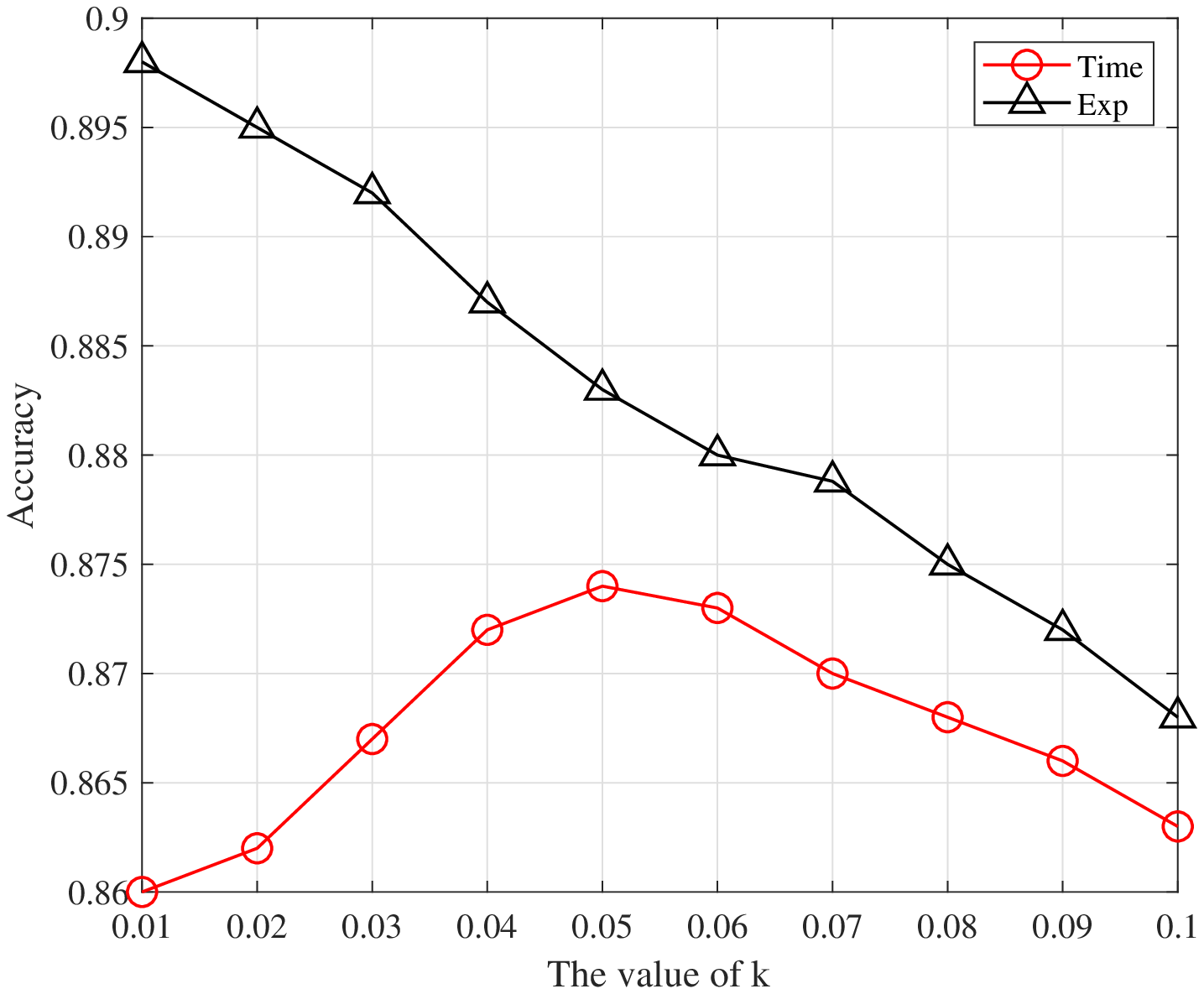}
}
\subfigure[Step and ANT-RDP]{\label{scale2}
  \includegraphics[width=0.22\textwidth]{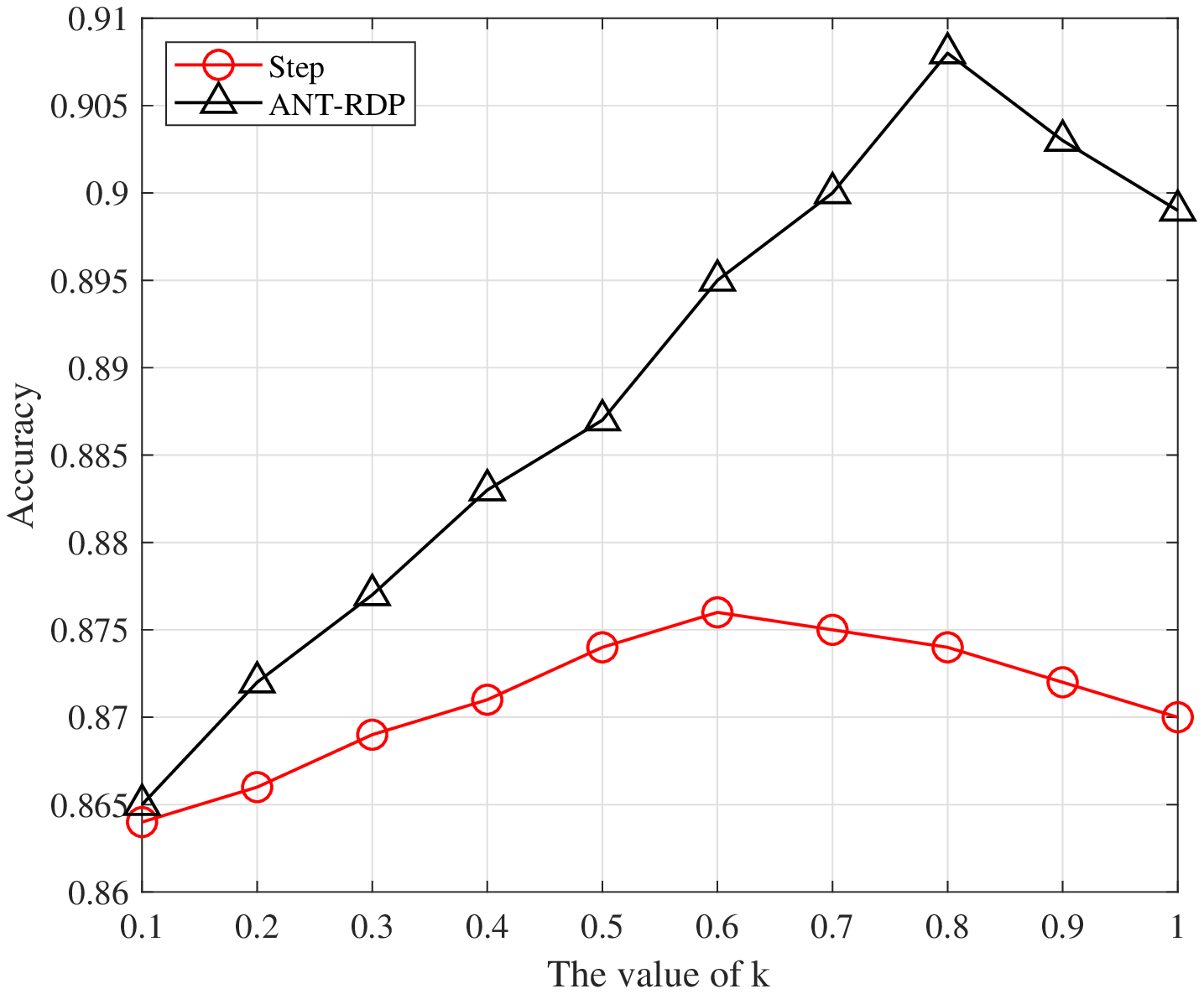}}
 \caption{The accuracy with different decay rate $k$ in the MNIST Dataset.} \label{scale}
\end{figure}
\begin{figure} \label{scale1}
\centering
\subfigure[Time and Exp]{\label{scale11}
  \includegraphics[width=0.22\textwidth]{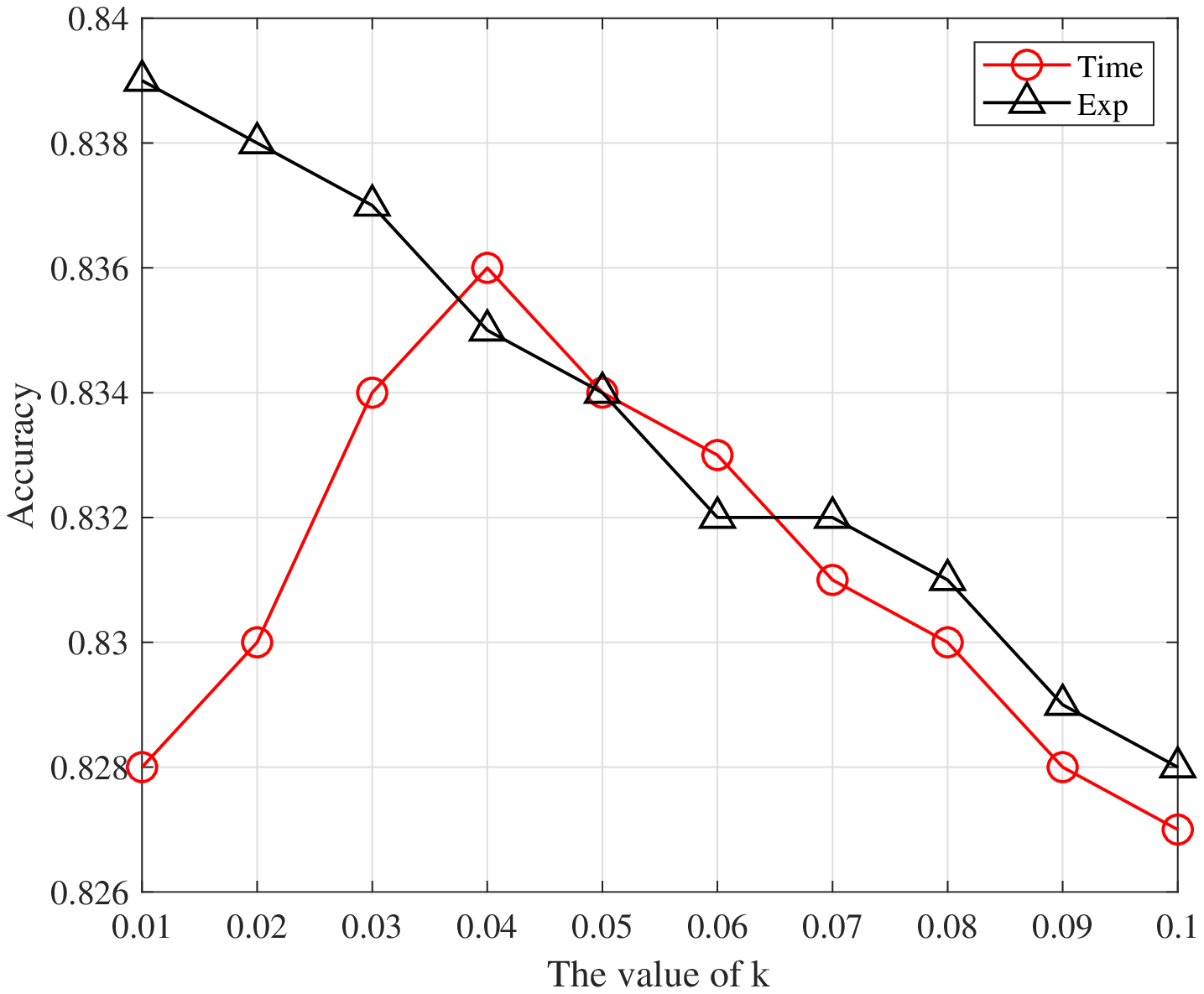}
}
\subfigure[Step and ANT-RDP]{\label{scale12}
  \includegraphics[width=0.22\textwidth]{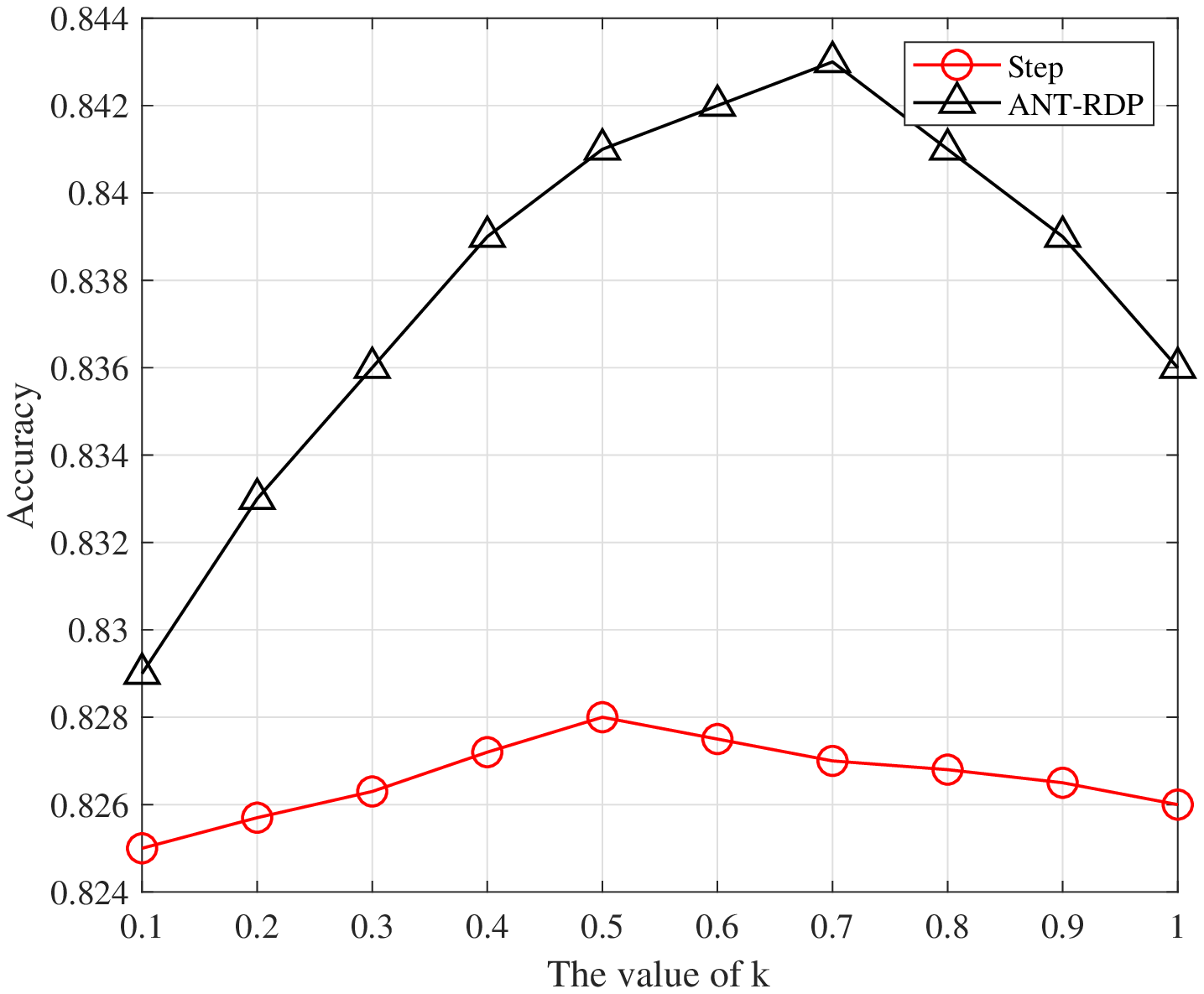}}
 \caption{The accuracy with different decay rate $k$ in the MNIST Dataset.} \label{scale1}
\end{figure}

From Fig.~\ref{scale} and Fig.~\ref{scale1}, we can find there exists an optimal decay rate in the current system setting. For example, in the ANT-RDP algorithm, the best accuracy occurs when $k=0.8$ in the Adult dataset, and $k=0.7$ in the Adult dataset, respectively. In addition, the performance of the Time, Step and ANT-RDP schedules seem to have a convex relationship with the value of $k$. Intuitively speaking, it may because that a lower decay rate causes to a longer training time, and it will produce a worse accuracy because most iterations will suffer from noise-added parameters.  On the other side, a higher decay rate leads to the noise scale decaying sharply in this case and therefore resulting in an insufficient training time, which may also degrades the accuracy. Moreover, we show a comparable results for these algorithms in Table.~\ref{difference} and Table.~\ref{difference1}.
\begin{table}
\centering
\caption{Average accuracy for dynamic schedules in the MNIST dataset: RDP-GAN ($\epsilon=0.5$); Time ($k=0.05$); Step ($k=0.6$); Exp ($k=0.01$); ANT-RDP ($k=0.8$)} \label{difference}
\resizebox{80mm}{5mm}{
\begin{tabular}{|c|c|c|c|c|c|c|}
  \hline
  {} & No Privacy & RDP-GAN& Time& Step& Exp& ANT-RDP  \\
  \hline
  Iterations&1000& 1000 & 657 & 564 & 902 & 868 \\
  \hline
  Accuracy & 0.995 & 0.863 & 0.874 & 0.876 & 0.898 & 0.908\\
  \hline
\end{tabular}}
\end{table}

\begin{table}
\centering
\caption{Average accuracy for dynamic schedules in the Adult dataset: RDP-GAN ($\epsilon=0.5$); Time ($k=0.04$); Step ($k=0.5$); Exp ($k=0.01$); ANT-RDP ($k=0.7$)} \label{difference1}
\resizebox{80mm}{5mm}{
\begin{tabular}{|c|c|c|c|c|c|c|}
  \hline
  {} & No Privacy & RDP-GAN & Time & Step & Exp & ANT-RDP   \\
  \hline
  Iterations&1000& 1000 &632 & 532 & 894 & 845 \\
  \hline
  Accuracy & 0.964 & 0.825 & 0.836 & 0.828 & 0.839 & 0.843\\
  \hline
\end{tabular}}
\end{table}
From the tables we can find that all the dynamic schedules can achieve a better performance than RDP-GAN, with smaller cost on the training iterations. The reason behind this is that the dynamic algorithm will adjust the scale of the added noise according to the testing performance, and it will cost more privacy budget when noise scale becomes smaller in each iteration, thus leading to less iterations. Moreover, the proposed ANT-RDP can obtain the highest accuracy on both datasets while it does not need the largest iterations. For example, compared with the Exp algorithm, the ANT-RDP has a better accuracy, i.e., there is a $0.4\%$ performance gain in the Adult dataset, and $0.9\%$ in the MNIST dataset, respectively, while needs less iterations, i.e., they are $49$ and $34$ less in each dataset.
\section{Related Work}
In this section, we provide a brief literature review of relevant topics: differential privacy and differentially private deep learning.
\subsection{Differential Privacy}
Differential privacy (DP) \cite{dwork2014algorithmic} and related algorithms have been widely studied in the literatures. The work \cite{dwork2006calibrating} laid the theoretical foundations of many DP studies by adding noise to make the maximum change of data related functions. Another related frameworks that adds noise on gradient are \cite{song2013stochastic} and its followed work \cite{song2015learning}, which studied to obtain information from models with DP requirements using stochastic gradient. In \cite{zhu2017differentially}, a comprehensive and structured overview of DP data publishing and analysis show several possible future directions and applications. In addition, to tightly analyze the composition, especially on the tails of the privacy loss, the work \cite{mironov2017renyi} gave a new privacy definition, named R\'{e}nyi DP. The R\'{e}nyi DP is not only a natural generalization of pure DP as it shares many properties with DP, but also has the capability to capture a more accurate aggregate privacy loss by the composition theorem.
\subsection{Differentially Private Deep Learning}
Abdadi et al. \cite{abadi2016deep} proposed a differentially private SGD algorithm for deep learning to offer provable privacy guarantees on the output model, and the work \cite{yu2019differentially} used the concentrated DP \cite{dwork2016concentrated}, which is in spirt of RDP, to provide cumulative privacy loss estimation over computations. Moreover, the applications of DP in GAN, such as \cite{xie2018differentially,8636556}, have well addressed information leakage by adding noises to gradients during training. They also can produce data samples with good quality under reasonable privacy budgets.

Different from these solutions, our work can obtain an more accurate estimation of privacy loss. The properties of RDP is first used to apply in the GAN. Moreover, we provide a naturally succeed method without artificial clipping on the noise-added gradients, which may destroy the learning performance.
\section{Conclusion}
In this work, we proposed a  R\'{e}nyi-differentially private-GAN (RDP-GAN), which achieves differential privacy (DP) under the framework of GAN by carefully adding random Gaussian noises on the value of the loss function during training. The advantages of applying perturbation on the loss function are three-fold: i) loss values are exchanged by the generator and discriminator directly which need more protection than parameters; ii) adding noise on the loss value can provide explicit privacy protection; and iii) it does not need value clipping on the parameters, so that the stability can be enhanced. We theoretically obtained the analytical results of the total privacy loss considering the subsampling method and multiple iterations of training with equal privacy budget allocation to each iteration. In addition, in order to alleviate the negative impact brought by the noise injection, we improved the proposed algorithm by designing an adaptive noise tuning algorithm, which can adaptively change the value of added noise according to the testing accuracy. Through comprehensive experimental results, we verified that the proposed algorithm can achieve a higher data utility compared with the existing DP-GAN algorithm at the same DP protection level.
\begin{appendices}
\section{Proof of Theorem~\ref{22}}
Let $F_d$ denote the loss function of discriminator, $X$ and $X'$ are adjacent datasets which only differ one element $x'$. Thus, the loss value of these two adjacent datasets after noise perturbation are expressed as:
\begin{equation} \label{loss_value}
\begin{split}
&L_d\left( {W;X} \right) = F_d\left( {W;X} \right) + \mathcal{N}(0,{\sigma ^2});\\
&L'_d\left( {W';X'} \right) = F_d'\left( {W';X'} \right) + \mathcal{N}(0,{\sigma ^2}),
\end{split}
\end{equation}
where $W$ denotes the parameter incurring in the neural network.

From Eq.~(\ref{loss_value}), we know $L_d\left( {W;X} \right)$ is a summation of two independent random variables. Using the Gaussian approximation approach in \cite{ding2017dna}, the summation of these two RVs can be treated as another RV which also follows a Gaussian distribution. Moreover, without loss of generality, we assume that $\sum_{i \in \left| X \right|\backslash [x']} {F}_d({x_i}) = 0$. Thus $L_d\left( {W;X} \right)$ and $L_d'\left( {W';X'} \right)$ are distributed identically except for the first coordinate and hence we have a one-dimensional problem \cite{abadi2016deep}. Let $u_0$ denote the pdf of $\mathcal{N}(0,\sigma^2)$ and let $u_1$ denote the pdf of $\mathcal{N}(\Delta S,\sigma^2)$. Thus we have
\begin{equation}
\begin{split}
&L'\left( {W';X'} \right) \sim v_0\triangleq u_0;\\
&L\left( {W;X} \right) \sim v_1\triangleq (1-q)\mu_0 + qu_1,
\end{split}
\end{equation}
where $q$ is the sampling rate and $\Delta S$ denotes the sensitivity bound on the loss function. In addition, $q=\frac{m}{n}$, where $m$, $n$ denotes the sampling and total size of dataset, respectively.

According to the definition of R\'{e}nyi differential privacy in Eq.~(\ref{RDP}), $\epsilon$ of one iteration can be expressed as\\
$\max \left( {{D_\alpha }\frac{{L_d\left( X \right)}}{{L_d\left( {X'} \right)}},{D_\alpha }\frac{{L_d\left( {X'} \right)}}{{L_d\left( X \right)}}} \right)$. For simplicity, we rewrite the \\$L_d\left( {W;X} \right)$ as $L_d\left( X \right)$.

To obtain $\epsilon$, we first calculate ${D_\alpha }\left( {\frac{{L_d(X)}}{{L_d(X')}}} \right)$ as
\begin{small}
\begin{equation}
\begin{split}
&{D_\alpha }\left( {\frac{{L_d(X)}}{{L_d(X')}}} \right) = \frac{1}{{\alpha  - 1}}\log {E_{x \sim L_d(X')}}{\left( {\frac{{L_d(X)}}{{L_d(X')}}} \right)^\alpha }\\
&= \frac{1}{{\alpha  - 1}}\log {\int {{u_0}\left[ {\frac{{\left( {1 - q} \right){u_0} + q{u_1}}}{{{u_0}}}} \right]} ^\alpha }dx \\
&= \frac{1}{{\alpha  - 1}}\log {\int {{u_0}\left\{ {1 - q + q\exp \left[ {\frac{1}{{2{\sigma ^2}}}\left( {2x\Delta S - \Delta {S^2}} \right)} \right]} \right\}} ^\alpha }dx \\
&\overset{(a)}= \frac{1}{{\alpha  - 1}}\log \int {{u_0}\sum\limits_{k = 0}^\alpha  {\left( {\begin{array}{*{20}{c}}
\alpha \\
k
\end{array}} \right)} {{\left( {1 - q} \right)}^{k - \alpha }}{q^k}\exp \left( \blacktriangle \right)} dx \\
&=\frac{1}{{\alpha  - 1}}\log \sum\limits_{k = 0}^\alpha  {\left( {\begin{array}{*{20}{c}}
\alpha \\
k
\end{array}} \right)} {\left( {1 - q} \right)^{k - \alpha }}{q^k}\exp \left[ {\frac{{k\Delta {S^2}(k - 1)}}{{2{\sigma ^2}}}} \right]\\
&\overset{(b)}\leq \frac{1}{{\alpha  - 1}}\log \sum\limits_{k = 0}^\alpha  {\left( {\begin{array}{*{20}{c}}
\alpha \\
k
\end{array}} \right)} {\left( {1 - q} \right)^{k - \alpha }}\left[ {q\exp \left( {\frac{{\Delta {S^2}\alpha }}{{2{\sigma ^2}}}} \right)} \right]\\
&=\frac{\alpha}{{\alpha  - 1}}\log {\left\{ {1 + q\left[ {\exp \left( {\frac{{\Delta {S^2}\alpha }}{{2{\sigma ^2}}}} \right) - 1} \right]} \right\} }\\
&\overset{(c)} \leq {\color{black}\frac{\alpha }{{\alpha  - 1}}q\left[ {\exp \left( {\frac{{\Delta {S^2}\alpha }}{{2{\sigma ^2}}}} \right) - 1} \right] }\\
&\overset{(d)} \leq {\color{black}\frac{\alpha }{{\alpha  - 1}}q\left[ {\frac{{\Delta {S^2}\alpha }}{{2{\sigma ^2}}} + O\left( {\frac{{\Delta {S^4}{\alpha ^2}}}{{4{\sigma ^4}}}} \right)} \right]}
\leq \frac{q\alpha^2\Delta S^2}{2(\alpha-1)\sigma^2},
\end{split}
\end{equation}
\end{small}
where step (a) uses the binomial expansion, $\blacktriangle$ is denoted by ${\frac{k}{{2{\sigma ^2}}}\left( {2x\Delta S - \Delta {S^2}} \right)}$, and step (b) is obtained by $k-1 \leq \alpha$. Step (c) uses Taylor series expansion,
and step (d) is obtained by Taylor series expansion.

We then calculate ${D_\alpha }\left(\frac{{L_d\left( {X'} \right)}}{{L_d\left( X \right)}}\right)$ as
\begin{small}
\begin{equation}
\begin{split}
&{D_\alpha }\left(\frac{{L_d\left( {X'} \right)}}{{L_d\left( X \right)}}\right)= \frac{{\rm{1}}}{{\alpha {\rm{ - 1}}}}\log {E_{x \sim L_d\left( X \right)}}\left[ {\frac{{L_d(X')}}{{L_d(X)}}} \right]\\
&= \frac{1}{{\alpha  - 1}}\log {\int {{u_0}\left[ {\frac{{{u_0}}}{{(1 - q){u_0} + q{u_1}}}} \right]} ^{\alpha  - 1}}dx\\
&= \frac{1}{{\alpha  - 1}}\log \int {{u_0}{{\left[ {\frac{{(1 - q){u_0} + q{u_1}}}{{{u_{\rm{0}}}}}} \right]}^{{\rm{1 - }}\alpha }}{\rm{dx}}} \\
&= \frac{1}{{\alpha  - 1}}\log {\int {{u_0}\left\{ {1 - q + q\exp \left[ {\frac{1}{{2{\sigma ^2}}}\left( {2x\Delta S - \Delta {S^2}} \right)} \right]} \right\}} ^{1-\alpha} }dx \\
& \overset{(a)} = \frac{1}{\alpha  - 1}\log \int {u_0}\left\{ \sum\limits_{k = 0}^\infty  ( - 1)^k\left( {\begin{array}{*{20}{c}}
{\alpha  + k - 2}\\
k
\end{array}} \right){q^k}\blacktriangledown \right\}dx \\
& \overset{(b)} \leq \frac{1}{{\alpha  - 1}}\log \sum\limits_{k = 0}^\infty  {{{( - 1)}^k}\left( {\begin{array}{*{20}{c}}
{\alpha  + k - 2}\\
k
\end{array}} \right){q^k}\exp \left( \blacktriangle \right)} {\left( {1 - q} \right)^{1 - k - \alpha }}\\
&=\frac{1}{{\alpha  - 1}}\log {\left[ {1 - q + q\exp \left( {\frac{{\Delta {S^2}\alpha }}{{2{\sigma ^2}}}} \right)} \right]^{1 - \alpha }}\\
&=- \log \left\{ {1 + q\left[ {\exp \left( {\frac{{\Delta {S^{\rm{2}}}\alpha }}{{{\rm{2}}{\sigma ^{\rm{2}}}}}} \right){\rm{ - 1}}} \right]} \right\}
<{D_\alpha }\left( {\frac{{L_d(X)}}{{L_d(X')}}} \right),
\end{split}
\end{equation}
\end{small}
where step (a) is obtained by ${(x + a)^{ - n}} = \\
\sum\limits_{k = 0}^\infty  {{{( - 1)}^k}\left( {\begin{array}{*{20}{c}}
{n + k - 1}\\
k
\end{array}} \right)} {x^k}{a^{ - n - k}}$, and $\blacktriangledown$ is denoted by \\
$\exp \left[ {\frac{k}{{2{\sigma ^2}}}\left( {2x\Delta S - \Delta {S^2}} \right)} \right] {{\left( {1 - q} \right)}^{1 - k - \alpha }}$. Step (b) is obtained by $k-1\leq \alpha$, and $\blacktriangle$ is denoted by $\left[ {\frac{k}{{2{\sigma ^2}}}\left( {2x\Delta S - \Delta {S^2}} \right)} \right]$.

As a result, $\epsilon=\frac{q\alpha^2\Delta S^2}{2(\alpha-1)\sigma^2}$ which concludes the proof.

\section{Proof of Lemma~\ref{33}}
By applying Proposition~\ref{comp} to their composition, we can obtain that for all $\alpha > 1$,
\begin{equation}
{D_\alpha }\left[ {F_d\left( X \right)||F_d\left( {X'} \right)} \right] \le 2\alpha n{\epsilon ^2}.
\end{equation}
Denote $\Pr \left[ {F_d(X) \in S} \right]$ by $P$ and  $\Pr \left[ {F_d(X') \in S} \right]$ by $Q$ and consider two cases.

Case I: $\log 1/Q \geq \epsilon^2n$. According to the Proposition~\ref{privacy} and choosing $\alpha=\sqrt{\log1/Q}/(\epsilon\sqrt{n})\geq 1$, we have
\begin{equation}
\begin{split}
P &\le {\left\{ {\exp \left\{ {{D_\alpha }\left[ {F_d\left( X \right)||F_d\left( {X'} \right)} \right]} \right\} \cdot Q} \right\}^{1 - 1/\alpha }}\\
&  \le \exp \left[ {2\left( {\alpha  - 1} \right)n{\epsilon ^2}} \right] \cdot {Q^{1 - 1/\alpha }}\\
&  \overset{(a)}< \exp \left[ {2\epsilon \sqrt {n\log 1/Q}  - 2n{\epsilon ^2} - \left( {\log Q} \right)/\alpha } \right] \cdot Q\\
& < \exp \left( {2\epsilon \sqrt {n\log 1/Q} } \right) \cdot Q,
\end{split}
\end{equation}
where step (a) is obtained by substituting $\alpha=\sqrt{\log1/Q}/(\epsilon\sqrt{n})
\geq 1$ into the equation.

Case II: $\log 1/Q < \epsilon^2n$. This case follows trivially, since the right hand of Eq.~(\ref{11111}) is larger than 1 as:
\begin{equation}
\exp \left( {2\epsilon \sqrt {n\log 1/Q} } \right) \cdot Q \ge \exp \left( {2\log 1/Q} \right) \cdot Q = 1/Q > 1.
\end{equation}
\section{Proof of Theorem~\ref{44}}
Let $X$ and $X'$ be two adjacent inputs, and $S$ be some subset results of the loss function $F_d$. To argue $(\epsilon_{{\emph{d}}},\delta)$-differential private of $F_d$, we need to verify that
\begin{equation}
\Pr \left[ {F_d(X) \in S} \right] \le \exp \left( \epsilon_{{\textrm{d}}}  \right)\Pr \left[ {F_d\left( {X'} \right) \in S} \right] + \delta.
\end{equation}
According to Lemma~\ref{33}, we know that
\begin{equation}
\begin{split}
&\Pr \left[ {F_d(X) \in S} \right] \le \\
&\exp \left\{ {2\epsilon \sqrt {n\log /\Pr \left[ {F_d(X') \in S} \right]} } \right\} \cdot \Pr \left[ {F_d\left( {X'} \right) \in S} \right].
\end{split}
\end{equation}
Denote $\Pr \left[ {F_d(X) \in S} \right]$ by $P$ and  $\Pr \left[ {F_d(X') \in S} \right]$ by $Q$ and we can obtain that
\begin{equation} \label{did}
P \le \exp \left( {2\epsilon \sqrt {n\log 1/Q} } \right) \cdot Q.
\end{equation}
Eq.~(\ref{did}) can be divided into two cases as follows.

Case I: $8\log 1/\delta  > \log 1/Q$. Then Eq.~(\ref{did}) can be derived by
\begin{equation}
\begin{split}
P &\le \exp \left( {2\epsilon \sqrt {8n\log 1/\delta } } \right) \cdot Q\\
& = \exp \left( {4\epsilon \sqrt {2n\log 1/\delta } } \right) \cdot Q
\leq \exp({\epsilon_{\textrm{d}}})+\delta.
\end{split}
\end{equation}

Case II: $8\log 1/\delta  \leq \log 1/Q$. Then Eq.~(\ref{did}) can be derived by
\begin{equation}
\begin{split}
P & \overset{(a)} \le \exp \left( {2\sqrt {\log 1/\delta  \cdot \log 1/Q} } \right) \cdot Q \\
& \le \exp \left( {\sqrt {1/2} \log 1/Q} \right) \cdot Q\\
&=Q^{1-1/\sqrt{2}}\leq Q^{1/8}=\delta
\leq \delta+\exp({\epsilon_{\textrm{d}}}),
\end{split}
\end{equation}
where step (a) is obtained by substituting the condition $8\log 1/\delta  \leq \log 1/Q$ into the equation.

This concludes the proof.
\end{appendices}
\bibliographystyle{IEEEtran}
\bibliography{rdp-gan}
\end{document}